\documentclass{article}

 \usepackage[preprint]{neurips_2026}

\usepackage[utf8]{inputenc} 
\usepackage[T1]{fontenc}    
\usepackage{hyperref}       
\usepackage{url}            
\usepackage{booktabs}       
\usepackage{amsfonts}       
\usepackage{nicefrac}       
\usepackage{microtype}      
\usepackage{xcolor}         
\usepackage{graphicx}
\usepackage{wrapfig}
\usepackage{enumitem}
\usepackage[table]{xcolor}

\usepackage{amsmath}

\usepackage{caption}     

\usepackage{amsmath,amssymb,amsthm,mathtools}
\usepackage{algorithm,algpseudocode}
\usepackage{bm}

\newtheorem{theorem}{Theorem}
\newtheorem{proposition}{Proposition}
\newtheorem{corollary}{Corollary}
\newtheorem{lemma}{Lemma}
\theoremstyle{definition}

\theoremstyle{remark}
\newtheorem{remark}{Remark}

\DeclareMathOperator{\Var}{Var}

\newcommand{\E}{\mathbb{E}}
\newcommand{\R}{\mathbb{R}}
\newcommand{\PQ}{\mathrm{PQ}}
\newcommand{\Lum}{\mathrm{Lum}}
\newcommand{\Wone}{W_{\!1}}
\newcommand{\xhat}{\hat{x}_0}
\newcommand{\Dec}{\mathrm{Dec}}

\usepackage{booktabs}
\usepackage[table]{xcolor}
\usepackage{colortbl}

\definecolor{rowgray}{HTML}{F1F3F5}        
\definecolor{accent}{HTML}{E8F1FB}         
\definecolor{categorytext}{HTML}{4A5560}   

\usepackage{tocloft}
\title{\textbf{\textbf{LumaGuide}}: Distribution Shaping for Training-Free HDR Generation in Diffusion Models}

\author{
Bowen Chen\textsuperscript{1}\thanks{Project lead and corresponding author. Contact: \texttt{bwchen@utexas.edu.}}\space\space\space 
Shreshth Saini\textsuperscript{1, 2}\thanks{Work done while Shreshth Saini at UT Austin.}\space\space\space 
Balu Adsumilli\textsuperscript{2}\space\space\space 
Alan C. Bovik\textsuperscript{1, 3}\space\space\space 
\vspace{2mm} \\
\textsuperscript{1}The University of Texas at Austin\qquad 
\textsuperscript{2}Google/YouTube\qquad 
\textsuperscript{3}University of Colorado Boulder \vspace{2mm}\\
\footnotesize
\texttt{https://github.com/bwchen05/LumaGuide}
}

\begin{document}
\maketitle

\begin{center}
  \includegraphics[width=\textwidth]{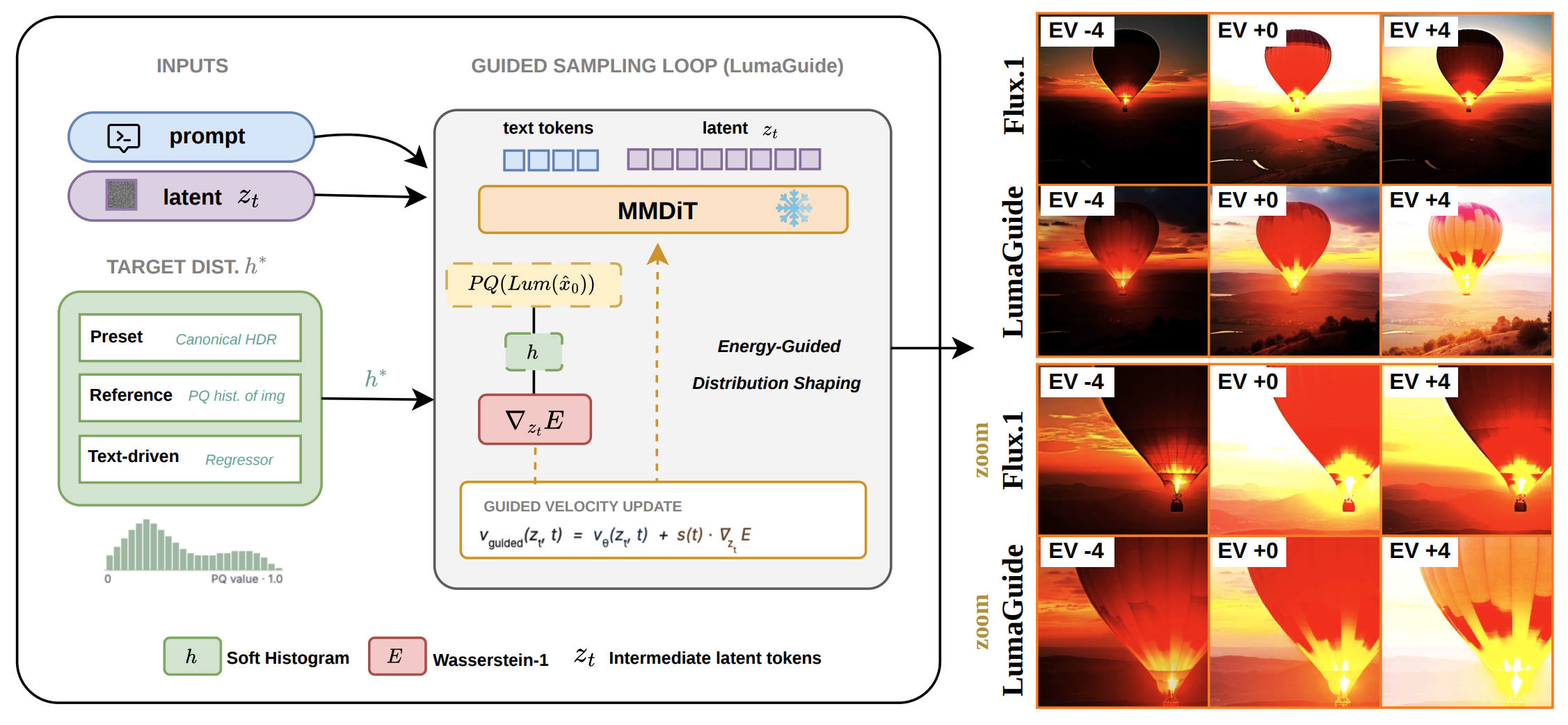}
  \captionof{figure}{\textbf{\textbf{LumaGuide}} steers a pretrained diffusion model toward HDR-consistent luminance distributions at sampling time, without modifying any model weights. Generated outputs preserve semantic content while producing structured highlights and faithful shadow detail under exposure adjustment, matching target HDR statistics in PQ space.}
  \label{fig:teaser}
\end{center}

\maketitle

\begin{abstract}

Pretrained diffusion models generate realistic images but are constrained by the statistical biases of their training data, limiting their ability to produce high dynamic range (HDR) content. In this work, we introduce \textbf{LumaGuide}, a training-free framework for distribution shaping in diffusion models. Instead of modifying model parameters, \textbf{LumaGuide} steers the sampling process to match target feature distributions via differentiable energy-based guidance. We instantiate this framework for HDR generation by controlling luminance distributions in perceptually uniform PQ space. Our results show that aligning luminance histograms is sufficient to induce HDR-consistent behavior, including coherent highlights and preserved shadow detail, while maintaining semantic fidelity. Beyond HDR, \textbf{LumaGuide} enables flexible specification of target distributions through data-driven presets, reference images, or text-driven predictors, and extends naturally to video generation with temporal consistency constraints. More broadly, our work demonstrates that controllable generation can be achieved by directly shaping output distributions at sampling time, without retraining diffusion models.
\end{abstract}

\section{Introduction}

\setlength{\intextsep}{1pt}%
\setlength{\abovecaptionskip}{1pt}%
\setlength{\belowcaptionskip}{2pt}%
\begin{wrapfigure}{r}{0.5\linewidth}
  \centering
  \includegraphics[width=\linewidth]{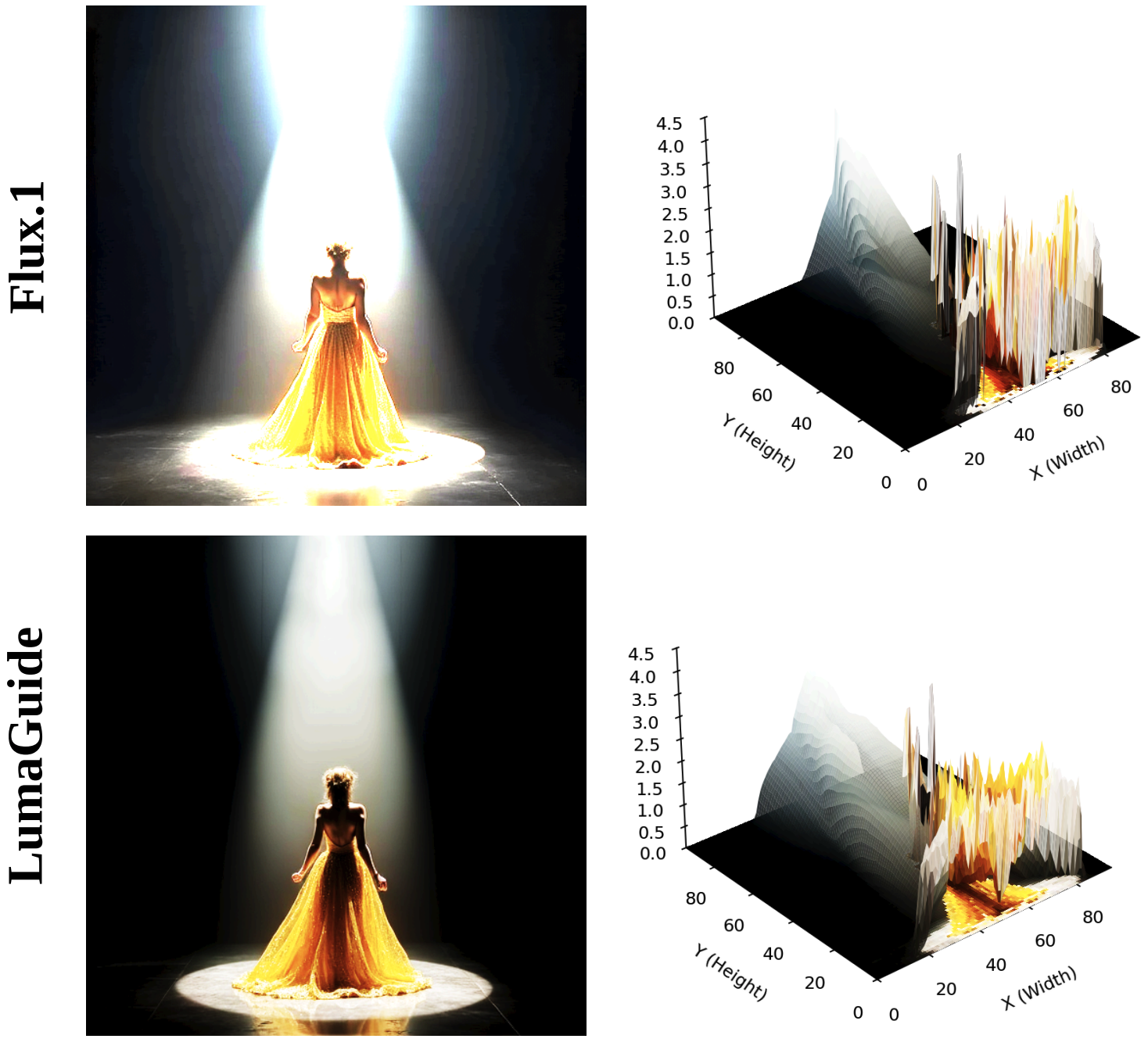}
  \caption{\textbf{Per-pixel PQ-luminance surface for an identical seed and prompt.} The Flux.1~\cite{flux} baseline (top) overshoots into clipped highlights; LumaGuide (bottom) redistributes mass into mid-tones and preserves the highlight gradient, matching the target HDR distribution.}
  \label{fig:lum_surface}
\end{wrapfigure}

Pretrained diffusion models~\cite{ho2020denoising, song2020score, rombach2022high, saharia2022photorealistic, ramesh2022hierarchical} have demonstrated remarkable abilities to generate realistic images from text prompts, yet their outputs remain implicitly constrained by the statistical properties of their training data. In particular, models trained on internet-scale corpora of images that are typically quantized and compressed for legacy 8-bit displays internalize a Standard Dynamic Range (SDR) prior. This limits their ability to represent the inherently wide dynamic range of natural light, which spans several orders of magnitude in luminance~\cite{reinhard2010high, itur2025bt2100}. As a result, even highly realistic text-to-image diffusion models often fail to synthesize physically plausible high-intensity highlights or coherent HDR luminance structure. Existing approaches to HDR generation typically rely on additional reconstruction pipelines or modifications to the generative model itself. Prior work has explored multi-stage reconstruction or exposure expansion pipelines~\cite{eilertsen2017hdr, marnerides2018expandnet, marnerides2021deep}, while more recent diffusion-based approaches operate in perceptually uniform domains such as PQ or PU21~\cite{itur2025bt2100, azimi2021pu21}. For example, X2HDR~\cite{wu2026x2hdr} shows that pretrained VAEs can faithfully encode HDR signals in perceptually uniform space, but still relies on fine-tuning the diffusion backbone, e.g., via LoRA~\cite{hu2022lora}, to overcome SDR-biased generation. Other diffusion-based HDR methods synthesize exposure brackets or fuse latent exposures to recover HDR content~\cite{bemana2025bracket, wang2025lediff, yu2026diffhdr}. These approaches are effective, but they often introduce additional training, inference, or architectural complexity, reducing flexibility across backbones and tasks.

In this work, we take a different perspective. Instead of modifying model parameters, we directly control the output statistics of a pretrained diffusion model at sampling time. We formalize this as distribution shaping, in which sampling is steered toward target output statistics via differentiable energy guidance. This perspective relates to guided sampling methods in diffusion models, such as classifier- and classifier-free guidance~\cite{dhariwal2021diffusion, ho2022classifier}, as well as broader training-free guidance methods that steer diffusion models using arbitrary differentiable objectives or energy functions~\cite{bansal2023universal, chung2022diffusion, yu2023freedom, park2023energy}. However, our method differs in that it shapes low-level perceptual distributions rather than semantic, class-conditional, or task-specific reconstruction signals. We instantiate this framework for text-to-HDR image generation. Given a text prompt, our goal is to synthesize an image whose luminance statistics are consistent with HDR appearance while preserving the semantic content and spatial coherence produced by the pretrained diffusion prior. We observe that luminance distributions in perceptually uniform PQ space capture a key aspect of HDR appearance~\cite{itur2025bt2100, mantiuk2011hdr, azimi2021pu21}. This motivates luminance histogram alignment as a natural objective for encouraging HDR-consistent rendering.

To this end, we introduce \textbf{LumaGuide}, a training-free framework for HDR distribution shaping in diffusion models. During sampling, \textbf{LumaGuide} computes a differentiable soft histogram of the predicted image in PQ space and minimizes a Wasserstein-1 distance ($W_1$)~\cite{arjovsky2017wasserstein} to the target distribution. The resulting gradient is backpropagated through the VAE to steer the diffusion trajectory in latent space. Since the histogram constraint is permutation-invariant, it does not explicitly impose spatial structure, allowing the diffusion prior to preserve geometry and semantics while adjusting global luminance statistics. This decoupling of semantic generation and distribution control also makes LumaGuide naturally applicable to video generation. By applying the same distribution shaping principle to pretrained video diffusion models~\cite{ho2022video, blattmann2023align, cogvideox}, we enable zero-shot HDR video synthesis. To maintain temporal consistency under distribution shaping, we introduce a Temporal Luminance Coherence (TLC) term that penalizes highlight flickering across frames, inspired by prior work on temporal consistency in video synthesis~\cite{lai2018learning}. This results in a training-free framework for controllable HDR video generation that maintains motion and structural fidelity. In summary, our contributions are as follows:
\vspace{-5pt}
\begin{itemize}[leftmargin=*]
    \item \textbf{Distribution Shaping Framework}: We propose a training-free framework for controlling output feature distributions in diffusion models via differentiable energy guidance at sampling time.

    \item \textbf{LumaGuide for Text-to-HDR Generation}: We instantiate this framework for HDR image synthesis, showing that luminance histogram alignment in perceptually uniform space provides an effective mechanism for inducing HDR-consistent outputs. We further extend the framework to video generation, introducing a TLC term to improve temporal stability without additional training.

    \item \textbf{Theoretical Guarantees}: We show that energy-guided distribution shaping admits controlled energy descent (Theorem~\ref{thm:descent}), that permutation-invariant feature constraints do not explicitly impose spatial rearrangements (Proposition~\ref{prop:decoupling}), and that global affine brightness adjustments cannot generally reproduce target HDR luminance distributions (Theorem~\ref{thm:nogo}).

    \item \textbf{Empirical Validation}: Extensive experiments across image and video backbones (e.g., Flux.1~\cite{flux}, SD3~\cite{sd3}, SDXL~\cite{sdxl}, and CogVideoX~\cite{cogvideox}) demonstrate significant improvements in distribution alignment and HDR fidelity over existing methods.

\end{itemize}
\section{Distribution Shaping in Diffusion Models}

We formalize \emph{distribution shaping} as the problem of steering a pretrained diffusion model so that a designated low-dimensional feature of its output matches a prescribed target distribution, without any modification to model weights. We then establish that this objective admits a principled energy-guided solution with provable guarantees on energy descent and on preservation of spatial structure.

\subsection{Problem Formulation}
\label{sec:formulation}
 
Let $p_\theta$ denote the implicit distribution of a pretrained flow-matching model with velocity field $v_\theta:\mathcal{Z}\times[0,1]\!\to\!\mathcal{Z}$ and decoder $\Dec:\mathcal{Z}\!\to\!\mathcal{X}$, where $\mathcal{X}=\R^{C\times H\times W}$.  Let $f:\mathcal{X}\!\to\!\R^{N}$ be a \emph{pixel-wise feature map}, i.e.\ $f(x)_i = \varphi(x_{\cdot,i})$ for some scalar map $\varphi$ acting on the channel vector at spatial location $i\in\{1,\dots,N\}$, $N=H\!\cdot\!W$. The \emph{empirical feature distribution} of $x\!\in\!\mathcal{X}$ under $f$ is the random measure (Empirical feature distribution):
\begin{equation}
P_f(x) \;=\; \tfrac{1}{N}\sum_{i=1}^{N} \delta_{f(x)_i} \;\in\; \mathcal{P}(\R).
\end{equation}

Given a target $P^{*}\!\in\!\mathcal{P}(\R)$ and a metric $\mathcal{D}$ on $\mathcal{P}(\R)$, distribution shaping seeks to reduce
\begin{equation}\label{eq:dsproblem}
\mathcal{E}(z) \;:=\; \mathcal{D}\!\bigl(P_f(\Dec(z)),\,P^{*}\bigr)
\end{equation}
along the sampling trajectory induced by the pretrained model, without modifying $v_\theta$ or $\Dec$.

Equation~\eqref{eq:dsproblem} differs sharply from posterior sampling or classifier guidance~\cite{chung2022diffusion, ho2022classifier, saini2025rectifiedcfgpp}: the constraint is imposed on a statistic of the output, rather than on its identity. This distinction allows us to retain the pretrained prior as the source of semantic and spatial structure while shaping marginal output behavior. In practice, $P_f(x)$ is approximated using a differentiable soft histogram.

\subsection{Energy-Guided Sampling}
\label{sec:energyguided}
 
We adopt a soft histogram approximation $h_\sigma(x)\!\in\!\R^{K}$ of $P_f(x)$, with bin centers $\{b_k\}_{k=1}^{K}$ and Gaussian kernel of bandwidth $\sigma$:
\begin{equation}\label{eq:softhist}
[h_\sigma(x)]_k \;=\; \frac{1}{N}\sum_{i=1}^{N}
   \exp\!\left(-\frac{(f(x)_i - b_k)^2}{2\sigma^{2}}\right)\,/\,Z_i,
\qquad
Z_i = \sum_{k'} \exp\!\left(-\tfrac{(f(x)_i-b_{k'})^2}{2\sigma^{2}}\right).
\end{equation}

The energy used in our framework is 
\begin{equation}\label{eq:energy}
E(x) \;=\; \mathcal{D}\!\bigl(h_\sigma(x),\,h^{*}\bigr),
\end{equation}

where $h^{*}$ is the target histogram and $\mathcal{D}$ is the Wasserstein-1 distance~\cite{arjovsky2017wasserstein} on $K$-bin distributions (Section~\ref{sec:lumaguide}).  

Because the clean output is unavailable during sampling, we evaluate $E$ on the flow estimate~\cite{chung2022diffusion,efron2011tweedie, kim2021noise2score} $\xhat(z_t) := \Dec\!\bigl(z_t - t\,v_\theta(z_t,t)\bigr)$. The guided velocity is
\begin{equation}\label{eq:guided}
v_{\text{g}}(z_t,t)
=
v_\theta(z_t,t)
+
s(t)\,\nabla_{z_t} E(\xhat(z_t)),
\end{equation}
with time-dependent guidance scale $s:[0,1]\to\R_{\ge 0}$. Since the sampler steps as $z_{t_{n+1}}=z_{t_n}+(t_{n+1}-t_n)v_{\mathrm g}$ with $t_{n+1}<t_n$, the added velocity term induces a negative-gradient step on $E$.

The next result establishes that this update is a descent method on $E$ in expectation, with a quantifiable noise term controlled by the diffusion model's own velocity error.

\begin{theorem}[Energy descent]
\label{thm:descent}
Assume $z\mapsto E(\xhat(z))$ is $L$-smooth and the velocity error has bounded second moment, $\E\|v_\theta-v^{*}\|^{2}\!\le\!\sigma_v^{2}$. Let $\eta_t:=|t_{n+1}-t_n|s(t)$ be the effective guidance step size. If $\eta_t\in(0,2/L]$, then
\begin{equation}
\E\bigl[\Delta E(z_t)\bigr]
\;\le\; -\,\eta_t\!\left(1-\tfrac{L\,\eta_t}{2}\right)\E\|\nabla_z E\|^{2}
\;+\; C\,\sigma_v^{2}\,t^{2}\,\Delta t,
\end{equation}
where $\Delta E(z_t):=E(\xhat(z_{t+\Delta t}))-E(\xhat(z_t))$ and $C$ is a constant depending only on the decoder. The deterministic term is maximized at $\eta_t=1/L$.
\end{theorem}

The first (negative) term is the guidance benefit, proportional to the squared gradient. The second (positive) term is the noise penalty, growing with the velocity-prediction error $\sigma_v^{2}$ and the timestep $t$. The energy decreases on average whenever the gradient signal exceeds the noise floor set by $\sigma_v^{2}t^{2}$. Full proof in Appendix.

\subsection{Why Distribution Shaping Preserves Semantics}
\label{sec:semantics}
 
The most striking empirical phenomenon in LumaGuide is that, despite applying a strong global energy gradient, semantic content and geometric layout remain largely intact. We now formalize the structural reason for this: such gradients do not explicitly encode spatial rearrangement, but act through local feature values and global histogram statistics.
 
\begin{proposition}[Spatial decoupling]
\label{prop:decoupling}
Let $E(x)=\mathcal{D}(h_\sigma(x),h^{*})$ for a pixel-wise feature $f(x)_i = \varphi(x_{\cdot,i})$.  Then:
\begin{enumerate}
\item[\textnormal{(i)}]\textbf{Permutation equivariance.} 

For any spatial permutation $\pi\!\in\!S_N$, $E(\pi\!\cdot\!x)=E(x)$ and $\nabla_xE(\pi\!\cdot\!x)=\pi\!\cdot\!\nabla_xE(x)$.
\item[\textnormal{(ii)}]\textbf{Pixel-wise gradient form.} The gradient at site $i$ depends only on the feature value $f(x)_i$ and on the global histogram $h_\sigma(x)$:
\begin{equation}
\bigl[\nabla_x E(x)\bigr]_{\cdot,i}
\;=\; \nabla_{x_{\cdot,i}} \varphi(x_{\cdot,i}) \cdot
       g\!\bigl(\varphi(x_{\cdot,i});\,h_\sigma(x),h^{*}\bigr),
\end{equation} for a scalar function $g$. 
\item[\textnormal{(iii)}]\textbf{No spatial ordering information.} $E$ depends on $x$ only through the multiset $\{f(x)_i\}_{i=1}^{N}$, and therefore cannot distinguish images that differ only by a spatial permutation.

\end{enumerate}
\end{proposition}

See Appendix for proof.

Algo.~\ref{alg:lumaguide} shows the full algorithm. It requires only a forward pass, the standard flow estimate, one backward pass through the decoder per step, and no additional model training.
 
\begin{algorithm}[t]
\caption{LumaGuide: Energy-Guided Distribution Shaping}
\label{alg:lumaguide}
\begin{algorithmic}[1]
\Require Pretrained flow model $v_\theta$, VAE decoder $\Dec$,
target histogram $h^{*}\!\in\!\Delta^{K-1}$, schedule $s(t)$,
discretization $\{t_n\}_{n=0}^{T}$ with $t_{n+1}<t_n$, soft-histogram bandwidth $\sigma$.
\State Sample $z_{t_0}\!\sim\!\mathcal{N}(0,I)$
\For{$n=0,\ldots,T-1$}
    \State $\xhat \gets \Dec\!\bigl(\, z_{t_n} - t_n\,v_\theta(z_{t_n},t_n)\,\bigr)$
           \Comment{flow decode (\(\xhat\) is RGB)}
   \State $Y \gets \PQ(\Lum(\xhat))$ \Comment{Per-pixel PQ luminance}
   \State $h_\sigma \gets$ soft-histogram of $Y$ via Eq.~\eqref{eq:softhist}
   \State $E \gets \Wone(h_\sigma,h^{*})$ \Comment{Wasserstein-1 distance, perceptually weighted}
   \State $g_t \gets \nabla_{z_{t_n}} E$ \Comment{Backprop through $\Dec$ and $\PQ\circ\Lum$}
   \State $z_{t_{n+1}} \gets z_{t_n} + (t_{n+1}-t_n)\!\bigl(\,v_\theta(z_{t_n},t_n) + s(t_n)\,g_t\,\bigr)$
\EndFor
\State \Return $\xhat(z_{t_T})$
\end{algorithmic}
\end{algorithm}

\section{LumaGuide for HDR Generation}
\label{sec:lumaguide}

We instantiate the framework using three deliberate design choices: PQ encoding, Wasserstein-1 distance, and a constant guidance schedule, supported by ablations in Section~\ref{sec:experiments} and the analysis below.

\subsection{PQ-Space Feature Map}
 
We take $\varphi = \PQ\circ \Lum$ where $\Lum(x) = 0.2627 R + 0.6780 G + 0.0593 B$ is BT.2020 luminance and $\PQ:\R_{\ge 0}\!\to\![0,1]$ is the SMPTE ST.2084 OETF.

\begin{proposition}[PQ-space gradient stability]
\label{prop:pqstab}
Let $L\!\in\![L_{\min},L_{\max}]$ denote scene luminance in nits with
$L_{\max}/L_{\min}\!\ge\!10^4$ (HDR range). For a fixed bin grid in PQ space,
\begin{enumerate}
\item[\textnormal{(i)}] $\PQ$ is monotone and smooth on $[L_{\min},L_{\max}]$, and $\PQ'(L)L$ is bounded above and below by positive constants. Hence per-pixel gradients converted through PQ are scale-balanced across the dynamic range.
\item[\textnormal{(ii)}] In linear space, the inverse PQ Jacobian induces polynomially mismatched scales across luminance ranges; for large $L$, the corresponding factor scales as $\Theta(L^{1-1/m_2})$, yielding gradient variance that grows as $\Omega(L_{\max}^{2(1-1/m_2)})$.
\item[\textnormal{(iii)}] Consequently, the noise-to-signal ratio of the histogram-shaping gradient in PQ space is bounded independently of $L_{\max}$, while in linear space it grows polynomially with the dynamic range.
\end{enumerate}
\end{proposition}
 
See Appendix for proof. This proposition explains the dominant effect of domain choice in Table~\ref{tab:ablation_design}: linear-space gradients exhibit scale mismatches across luminance ranges, making stable step-size selection difficult.

\subsection{Wasserstein-1 over KL and \texorpdfstring{$\ell_2$}{L2}}
 
The choice of distribution metric is important. HDR luminance histograms in PQ space are defined by a sparse high-percentile tail; a small fraction of pixels carry the highlights that distinguish HDR from SDR, and the optimization signal must reach those bins. Standard distributional metrics differ sharply in whether they can. On the real line, the Wasserstein-1 distance admits the closed-form CDF representation $\Wone(h,h^{*})=\sum_{k=1}^{K-1}(b_{k+1}-b_k)|F_k-F_k^{*}|$ \cite{arjovsky2017wasserstein}, where $F,F^{*}$ are the discrete CDFs of $h,h^{*}$ on the ordered grid $\{b_1\!<\!\cdots\!<\!b_K\}$. The subgradient $\partial\Wone/\partial h_k$ is therefore bounded and remains non-zero whenever the relevant CDF gaps are non-zero. Crucially, its magnitude depends on the bin spacing and the CDF gap, \emph{not} on the bin probabilities themselves, so even when $h_k\!\ll\!1$ in a sparse highlight bin, the gradient driving mass toward that bin remains well-conditioned. 

KL and $\ell_2$ fail in this regime, in opposite directions. The KL gradient $\partial\mathrm{KL}(h\|h^{*})/\partial h_k=\log(h_k/h_k^{*})+1$ can become unbounded near sparse bins, producing unstable updates where reliable tail statistics are most needed. The $\ell_2$ gradient $2(h_k-h_k^{*})$ depends only on pointwise bin differences and does not exploit the ordering of luminance bins, making it less effective at transporting mass toward sparse highlight regions. Neither metric respects the ordering of luminance values; distant bins are treated as exchangeable rather than as a transport problem on the line. Table~\ref{tab:ablation_design} bears this out empirically. 

\subsection{Guidance Schedule}
\label{sec:schedule}
 
Following previous methods~\cite{dhariwal2021diffusion,ho2022classifier,yu2023freedom,flux}, we adopt a constant guidance schedule $s(t)\!=\!s_0$ throughout sampling. We confirm this in our setting (Table~\ref{tab:ablation_schedule}), where the constant schedule outperforms an SNR-weighted bell-shaped schedule and a time-windowed variant on every reported metric. Appendix discusses this in more detail.

\subsection{Extension to Video Generation}
\label{subsec:extension2video}
The same distribution shaping principle extends to video by enforcing luminance constraints at each frame, but naive frame-wise guidance may introduce temporal inconsistencies, particularly in high-luminance regions. To reduce this effect, we introduce a Temporal Luminance Coherence (TLC) term that penalizes abrupt fluctuations in these regions. Let $h_\tau$ be the soft histogram of frame $\tau$, and let $\Omega_\tau$ be a high-luminance mask. We define
\begin{equation}
\label{def:tlc}
\mathcal{L}_{\mathrm{TLC}}
=
\sum_{\tau=2}^{T}
\bigl\|(Y_\tau-Y_{\tau-1})\odot \Omega_\tau\bigr\|_{2}^{2},
\qquad
E_{\mathrm{video}}
=
\sum_{\tau=1}^{T}\Wone(h_\tau,h^{*})
+
\lambda_{\mathrm{TLC}}\mathcal{L}_{\mathrm{TLC}}.
\end{equation}
 
\paragraph{Uniform target specification:}
Distribution shaping admits three target-specification modes $h^{*}\in\Delta^{K-1}$ within a single objective: 
(a) preset distributions, $h^{*}$ fixed offline;
(b) reference-based, $h^{*}\!=\!h_\sigma(x_{\mathrm{ref}})$;
(c) text-driven, $h^{*}\!=\!\mathrm{MLP}\!\circ\!\mathrm{CLIP}(c)$ for caption $c$.
All three plug into Algorithm~\ref{alg:lumaguide} without altering sampling. Since $\Delta^{K-1}$ is convex, their mixtures remain valid targets under the same objective.

\begin{figure}[t]
    \centering
    \includegraphics[width=\linewidth]{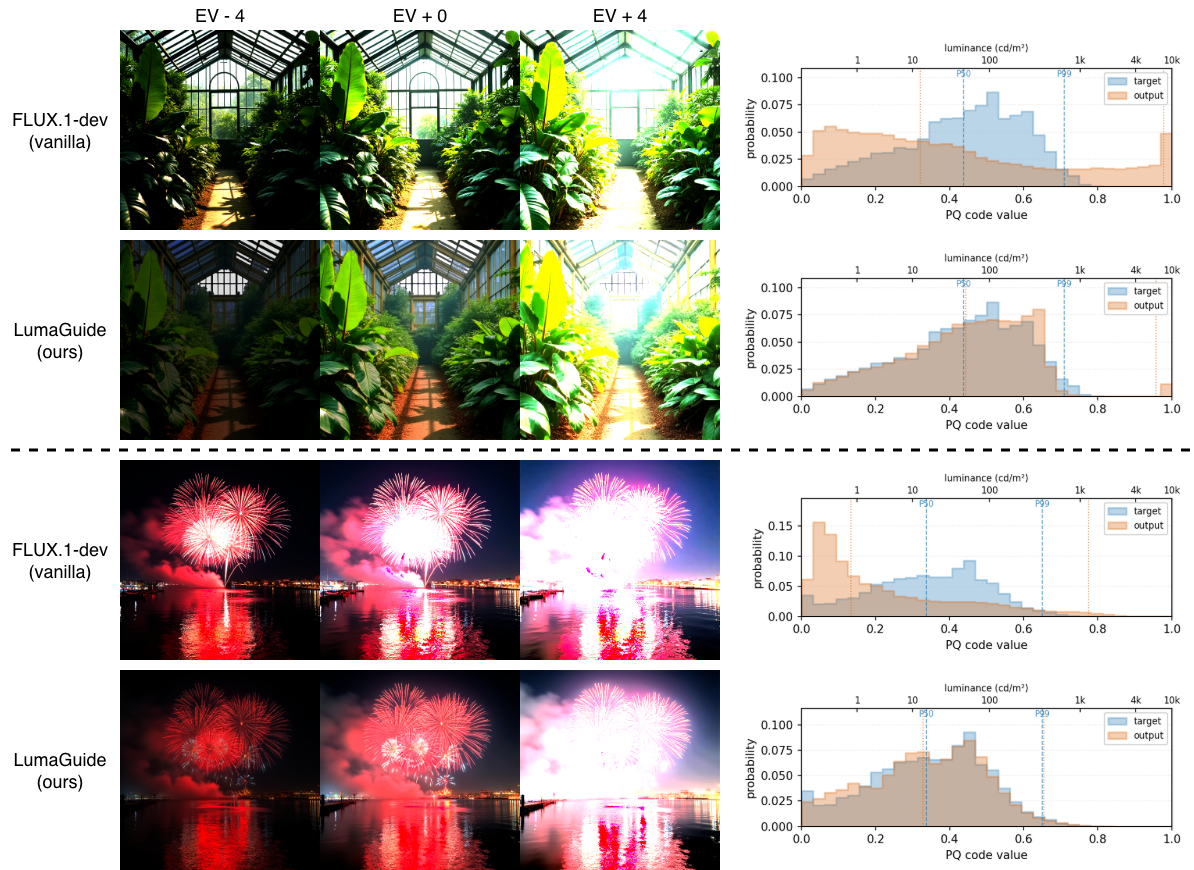}
    \caption{\textbf{HDR visualization under different exposure levels.} In PQ space, Flux.1~\cite{flux} often yields over-exposed, non-HDR luminance distributions. \textbf{LumaGuide} steers the luminance distribution toward a target HDR distribution while preserving scene structure.}
    \label{fig:hdr_results}
    \vspace{-6pt}
\end{figure}

\section{Experiments}
\label{sec:experiments}

Our primary experiments are conducted using Flux.1-dev~\cite{flux}, SD3~\cite{sd3}, and SDXL~\cite{sdxl}, without modifying model parameters or architectures. For video generation, we extend the same distribution shaping principle to CogVideoX~\cite{cogvideox}. More detailed results are provided in Appendices~\ref{app:baselines} and~\ref{app:video}. \textbf{LumaGuide} is plug-and-play for diffusion-based image and video models.

\subsection{Experimental Setup}

Unless otherwise specified, all experiments are performed at a resolution of $512 \times 512$. We use 28 sampling steps and a classifier-free guidance (CFG) scale of 3.5. All evaluations are conducted on a set of 100 HDR-oriented prompts covering diverse luminance conditions. Luminance distributions are represented using $K=32$ histogram bins in PQ space. We employ differentiable soft binning using a Gaussian kernel with standard deviation $\sigma = 0.5 / K$. Our default configuration uses a base guidance scale of $s_0 = 2000$, together with a perceptual-log Wasserstein-1 objective, a constant guidance schedule, and integral normalization. By default, target luminance distributions are obtained using the text-driven regressor described in Appendix.

For quantitative evaluation, we follow~\cite{wu2026x2hdr} and perform all comparisons in PQ-encoded space using Q-Eval~\cite{zhang2025q} as a proxy for measuring perceptual quality and alignment for HDR image generation. We use \(uW_1\) for unweighted Wasserstein-1 distance between generated and target PQ-luminance histograms over \(K=32\) bins, \(p50_{\mathrm{dist}}\) and \(p99_{\mathrm{dist}}\) for absolute percentile errors in PQ space, and \(\mathrm{nits}_{99}\) for the 99th-percentile luminance after inverse PQ conversion. 
\(\mathrm{DR}_{\mathrm{stops}}\) reports the robust luminance span in exposure stops, while JOD is the Bradley--Terry just-objectionable-difference score from the subjective study. We further discuss the details of HDR evaluation metrics in Appendix. All experiments are conducted on a single NVIDIA A100 GPU (40GB). For high-resolution generation ($1024 \times 1024$), we employ VAE gradient checkpointing and spatial tiling to enable memory-efficient backpropagation during guided sampling. Video experiments are conducted on a single NVIDIA H200 GPU.

\subsection{Distribution Control and HDR Emergence}

We first evaluate whether \textbf{LumaGuide} can effectively control luminance distributions and induce HDR-consistent rendering behavior. As shown in Figure~\ref{fig:hdr_results}, the proposed guidance consistently reshapes the generated PQ luminance histograms toward the target HDR distribution while preserving semantic structure and spatial coherence.

\begin{wraptable}{r}{0.5\linewidth}
\vspace{-6pt}
\centering
\small
\setlength{\tabcolsep}{7.5pt}
\caption{
\textbf{Effect of guidance scale $s_0$ on distribution alignment and luminance statistics.} We select $s_0 = 2000$ as the default trade-off point.
}
\begin{tabular}{lcccc}
\toprule
$s_0$ & uW1 $\downarrow$ & $p50_{\text{dist}} \downarrow$ & $p99_{\text{dist}} \downarrow$ & nits$_{99}$ \\
\midrule
0           & 3.79 & 0.116 & 0.207 & 4867 \\
10          & 3.77 & 0.115 & 0.206 & 4861 \\
100         & 3.42 & 0.104 & 0.198 & 4658 \\
500         & 2.16 & 0.071 & 0.153 & 3558 \\
1000        & 1.25 & 0.045 & 0.105 & 2648 \\
\rowcolor{green!12}
2000        & 0.58 & 0.024 & 0.053 & 1694 \\
3000        & 0.49 & 0.020 & 0.041 & 1425 \\
5000        & 0.52 & 0.019 & 0.045 & 1496 \\
\bottomrule
\end{tabular}
\label{tab:scale}
\end{wraptable}

Increasing the guidance scale generally improves distribution alignment, with the Wasserstein-1 distance decreasing substantially up to \(s_0=3000\) before slightly degrading at \(s_0=5000\) (Table~\ref{tab:scale}). Importantly, the guidance does not simply increase overall brightness. Instead, it redistributes luminance mass from over-saturated highlight regions toward mid-tones, resulting in more balanced luminance allocation and better preservation of highlight structure. This effect is particularly visible under exposure adjustment in Figure~\ref{fig:hdr_results}, where baseline generations often exhibit clipped highlights and unstable luminance behavior when interpreted in PQ space, while \textbf{LumaGuide} produces more coherent highlight roll-off and preserved shadow detail characteristic of HDR content.

At excessively large guidance scales, the optimization may become over-constrained, leading to artifacts such as banding and unnatural textures. Additional analysis of these failure cases is provided in Appendix. We select $s_0 = 2000$ as the default trade-off point.

\subsection{Comparisons with Existing Models}
Table~\ref{tab:baseline_comparison} compares \textbf{LumaGuide} with existing HDR generation methods. Although X2HDR~\cite{wu2026x2hdr} achieves slightly higher Q-quality, it relies on fine-tuning the diffusion backbone, whereas \textbf{LumaGuide} remains entirely training-free and preserves the pretrained model unchanged. Despite this, our method achieves the best Q-alignment score and the largest dynamic range among all compared approaches, indicating stronger HDR characteristics and more faithful luminance allocation.

These improvements are also reflected perceptually. Figure~\ref{fig:subjective} shows that \textbf{LumaGuide} achieves the highest preference in our JOD-based subjective study, outperforming all baselines in overall HDR quality. Qualitative comparisons in Figure~\ref{fig:baseline_comparison} further demonstrate improved highlight structure and shadow preservation under different exposure settings.

In addition, \textbf{LumaGuide} remains computationally efficient, introducing only moderate overhead over vanilla Flux and runtime comparable to X2HDR~\cite{wu2026x2hdr}, while remaining substantially faster than BracketDiffusion~\cite{bemana2025bracket}.

\begin{table}[!htbp]
    \centering
    \small
    \setlength{\tabcolsep}{12.2pt}
    \caption{
    \textbf{Comparison with HDR generation baselines.} \textbf{LumaGuide} achieves the best alignment and dynamic range with competitive quality and moderate runtime.
    }
    \begin{tabular}{lccccc}
    \toprule
    Method & Q-quality $\uparrow$ & Q-alignment $\uparrow$ & $\mathrm{DR}_{\text{stops}}$ $\uparrow$ & JOD $\uparrow$ & Time $\downarrow$ \\
    \midrule
    LEDiff~\cite{wang2025lediff}             & 0.425 & 0.612 & 4.71  & -0.88 & $\sim$8.6~s \\
    BracketDiffusion~\cite{bemana2025bracket}    & 0.448 & 0.648 & 12.25 & -0.30 & $\sim$389~s \\
    X2HDR~\cite{wu2026x2hdr}              & \textbf{0.579} & 0.773 & 11.41 & +0.43 & \textbf{$\sim$6~s} \\
    \rowcolor{green!12}
    \textbf{LumaGuide}  & 0.568 & \textbf{0.814} & \textbf{14.99} & \textbf{+0.75} & 7.8~s \\
    \bottomrule
    \end{tabular}
    \label{tab:baseline_comparison}
\end{table}

\begin{figure}[t]
\centering
\includegraphics[width=\linewidth]{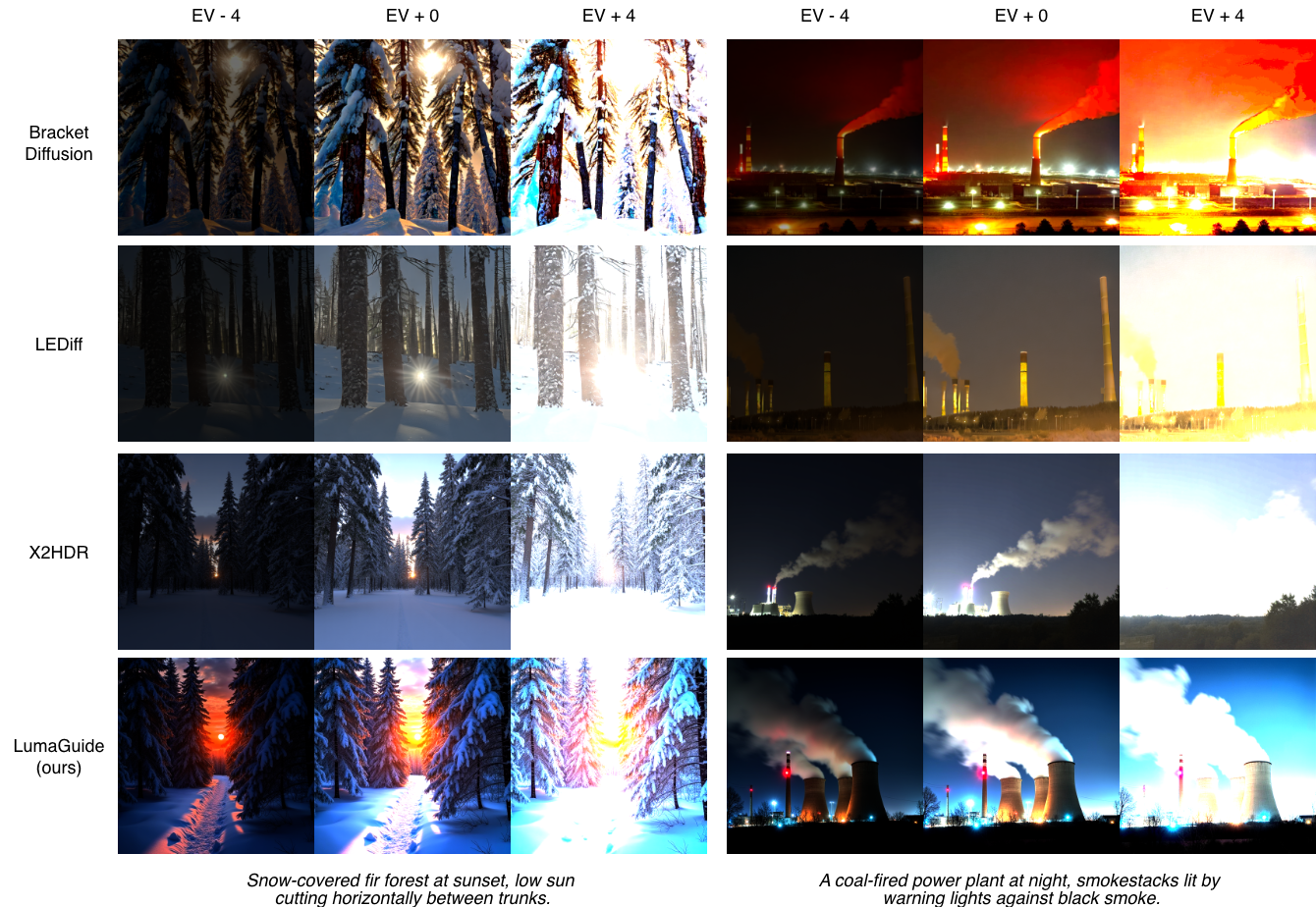}
\caption{
Qualitative comparison with existing HDR generation methods. For fair comparison, we normalize the exposure of all methods by adjusting the EV+0 images such that the median luminance is fixed at 8 nits since several existing methods are not calibrated to absolute luminance values, making direct visual comparison otherwise unreliable.
}
\label{fig:baseline_comparison}
\end{figure}

\begin{table}[!t]
\centering
\small
\setlength{\tabcolsep}{12.2pt}
\caption{
\textbf{Ablation of feature domain and distribution distance.} PQ-space guidance significantly improves distribution alignment over linear-domain variants. Within PQ space, Wasserstein-1 (W$_1$) further outperforms $\ell_2$ and KL by enabling ordered transport across histogram bins.
}
\label{tab:ablation_design}
\begin{tabular}{lccc ccc}
\toprule
\textbf{Setting} & Domain & Distance & uW1 $\downarrow$ & $p50_{\text{dist}} \downarrow$ & $p99_{\text{dist}} \downarrow$ & $\mathrm{DR}_{\text{stops}}$ $\uparrow$ \\
\midrule
Linear + W$_1$        & Linear & W$_1$      & 3.73 & 0.115 & 0.199 & 16.37 \\
Linear + $\ell_2$     & Linear & $\ell_2$   & 3.79 & 0.116 & 0.207 & 16.33 \\
PQ + $\ell_2$         & PQ     & $\ell_2$   & 3.40 & 0.100 & 0.206 & 16.18 \\
PQ + KL               & PQ     & KL         & 2.06 & 0.065 & 0.143 & \textbf{16.51} \\
\rowcolor{green!12}
PQ + W$_1$ & PQ & W$_1$ & \textbf{0.58} & \textbf{0.024} & \textbf{0.053} & 14.99 \\
\bottomrule
\end{tabular}
\end{table}

\subsection{Analysis and Ablations}

We analyze key design choices in \textbf{LumaGuide} to understand whether the observed improvements arise from principled distribution shaping rather than trivial transformations. More detailed ablation studies can be found in Appendix. 

\vspace{-8pt}
\paragraph{Effect of domain and distance.}

Table~\ref{tab:ablation_design} studies the impact of feature domain (linear vs.\ PQ) and distance metric (W$_1$, $\ell_2$, KL). Linear-domain guidance performs similarly to the unguided baseline, indicating that linear luminance does not provide perceptually meaningful gradients for distribution shaping. In contrast, PQ-space guidance consistently improves alignment, even with simple $\ell_2$ distance, highlighting the importance of perceptual reparameterization. Within PQ space, W$_1$ further outperforms both $\ell_2$ and KL divergence, likely because it preserves the ordering of histogram bins and enables smoother luminance transport across neighboring bins.

\vspace{-8pt}
\paragraph{Highlight-aware distribution shaping.}
Table~\ref{tab:ablation_weighting} compares different histogram weighting strategies for highlight control. Manual weighting provides moderate alignment but limited perceptual quality. Dual-histogram weighting aggressively emphasizes highlight bins but introduces noticeable degradation in distribution alignment. In contrast, the proposed perceptual-log weighting achieves the best balance, producing the lowest \(p99_{\text{dist}}\) among weighting strategies, together with the highest perceptual quality and alignment scores. This suggests that logarithmic damping stabilizes optimization while preserving sensitivity to high-luminance regions.

\begin{figure}[t]
\centering

\begin{minipage}[b]{0.42\linewidth}
  \centering
  \small
  \setlength{\tabcolsep}{3.5pt}
  \renewcommand{\arraystretch}{1.1}
  \begin{tabular}{l c c >{\columncolor{green!12}}c}
  \toprule
  \textbf{Metric} & Manual & Dual-hist. & \textbf{Perc.-log} \\
  \midrule
  uW1 $\downarrow$                       & 0.70  & 1.99           & \textbf{0.58}  \\
  $p99_{\text{dist}}\downarrow$          & 0.055 & 0.206          & \textbf{0.053} \\
  $\mathrm{DR}_{\text{stops}}\uparrow$   & 14.38 & \textbf{16.84} & 14.99          \\
  Q-quality $\uparrow$                   & 0.514 & 0.325          & \textbf{0.568} \\
  Q-alignment $\uparrow$                 & 0.708 & 0.299          & \textbf{0.814} \\
  \bottomrule
  \end{tabular}
  \vspace{1.5mm}
  \captionof{table}{\textbf{Ablation of histogram bin weighting strategies.} Perceptual-log weighting gives the best results.}
  \label{tab:ablation_weighting}
\end{minipage}%
\hfill
\begin{minipage}[b]{0.53\linewidth}
  \centering
  \includegraphics[width=0.87\linewidth]{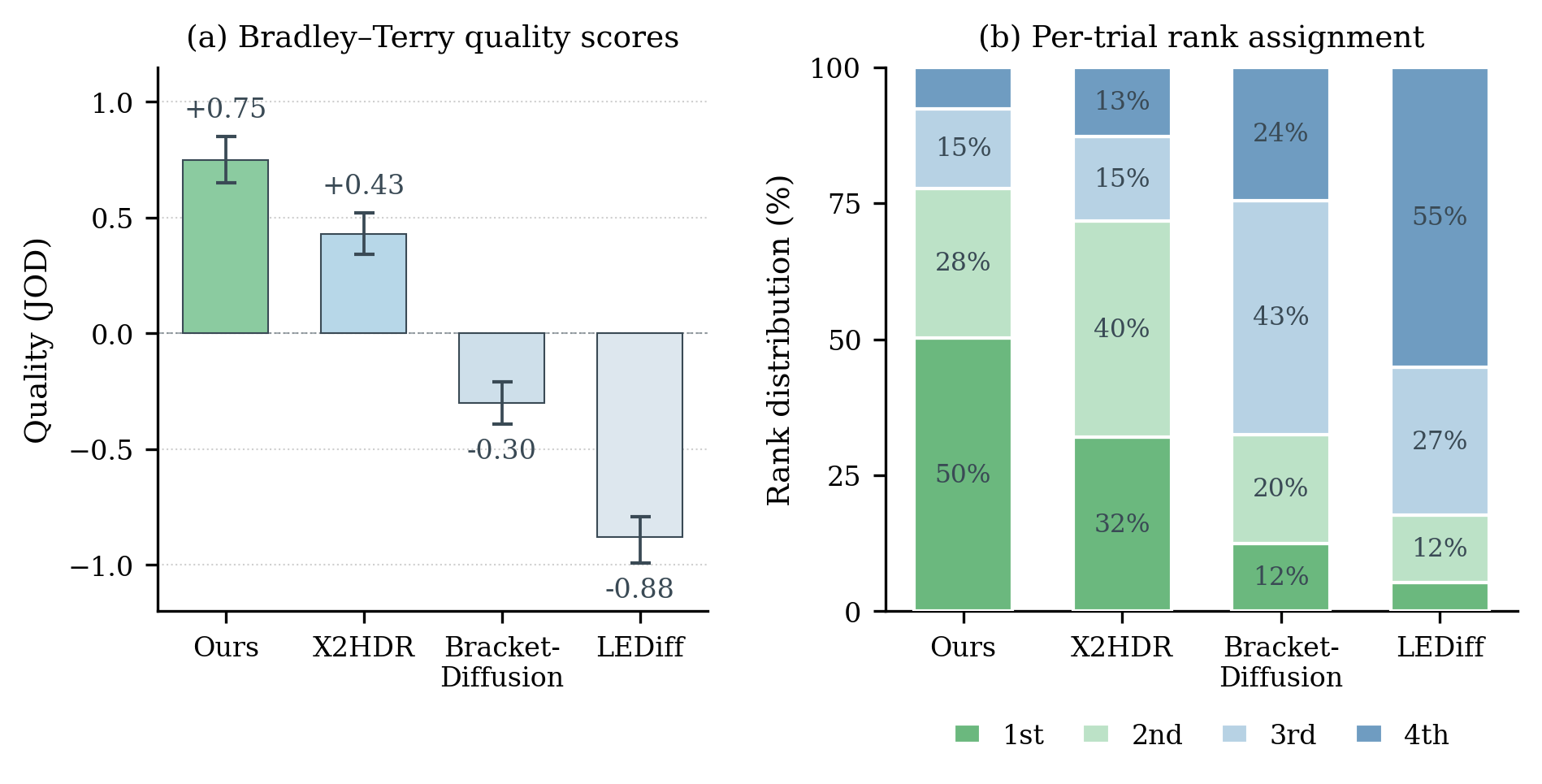}

  \captionof{figure}{\textbf{Subjective study}.
  \textbf{Left:} Bradley--Terry JOD scores with 95\% CIs.
  \textbf{Right:} Per-trial rank distribution. LumaGuide is ranked first in \(50\%\) of trials.}
  \label{fig:subjective}
\end{minipage}
\vspace{-2mm}
\end{figure}

\subsection{Flexible Control and Video Extension}

\begin{wraptable}{r}{0.48\textwidth}
\centering
\small
\setlength{\tabcolsep}{1.5pt}
\caption{\textbf{Cross-backbone comparison of LumaGuide} at the selected operating points for each model. }
\label{tab:all_baseline_results1}
\begin{tabular}{l c c c}
\toprule
\textbf{Model} & Q-quality $\uparrow$ & Q-alignment $\uparrow$ & DR (stops) $\uparrow$ \\
\midrule
\textbf{Flux} & \textbf{0.568} & \textbf{0.814} & 14.99 \\
SD3                  & 0.512          & 0.795          & 15.42 \\
SDXL                 & 0.431          & 0.655          & \textbf{15.90} \\
\bottomrule
\end{tabular}
\end{wraptable}

A key advantage of \textbf{LumaGuide} is that the target distribution \(\mathcal{P}^*\) can be specified externally. We support three modes: (1) user-driven presets, where expert users directly define target luminance histograms over PQ bins; (2) reference-based specification, where the PQ luminance histogram is extracted from a reference image; and (3) text-driven specification, where a lightweight regressor maps text descriptions generated by Qwen2.5-7B-Instruct~\cite{qwen25} to luminance histograms derived from the Beyond8Bits dataset~\cite{Saini_2026_Beyond8Bits, Saini_2025_ICIP_CHUG, Saini_2026_WACV_BrightRate}. The regressor is used solely for target prediction and does not modify the diffusion backbone; implementation details are provided in Appendix.
Importantly, \textbf{LumaGuide} is model-agnostic and can be applied to pretrained diffusion models without retraining. Table~\ref{tab:all_baseline_results1} shows results on Flux.1~\cite{flux}, SD3~\cite{sd3}, and SDXL~\cite{sdxl}, demonstrating consistent luminance distribution control and HDR characteristics across backbones. We further extend \textbf{LumaGuide} to video generation by applying the same distribution shaping objective independently to each frame. To mitigate temporal luminance instability, we incorporate the Temporal Luminance Coherence (TLC) term introduced in Section~\ref{subsec:extension2video}. Additional image and video results are provided in Appendix.

\section{Conclusion}
\label{sec:conclusion}

We presented LumaGuide, a training-free framework for distribution shaping in diffusion models. By steering the sampling process toward target feature distributions, our method enables direct control over output statistics without modifying model parameters. Instantiated for HDR generation, LumaGuide shows that shaping luminance distributions can induce HDR-consistent behavior while preserving semantic and spatial fidelity. More broadly, our results suggest that controllable generation can be achieved by directly shaping output distributions at sampling time. Extending this framework to richer spatially structured or multi-modal feature distributions remains an important direction for future work.

\bibliographystyle{plain}
\bibliography{reference}

\newpage
\appendix

\clearpage
\appendix
\onecolumn

\tableofcontents
\clearpage

\begin{center}
  {\LARGE\bfseries Appendix}
\end{center}

\vspace{1em}

\section{Appendix}

\subsection{Related Work}

\subsubsection{HDR Generation Methods}

Image generation has seen rapid progress with deep generative models. Early approaches relied on GAN-based methods~\cite{goodfellow2020generative,karras2020analyzing} and autoregressive models~\cite{ramesh2021zero,yu2022scaling} to synthesize realistic images, but these models often struggled with training stability, likelihood-quality trade-offs, or scalability. More recently, diffusion models~\cite{ho2020denoising,song2020score} have emerged as the dominant paradigm for high-fidelity image synthesis, achieving strong performance in class-conditional and text-to-image generation~\cite{dhariwal2021diffusion,rombach2022high,saharia2022photorealistic,ramesh2022hierarchical}.

Extending these models to HDR generation presents additional challenges, as pretrained diffusion models are inherently biased toward SDR data distributions. Prior work on HDR reconstruction and inverse tone mapping expands dynamic range from SDR inputs using supervised learning, saturated-region reconstruction, or exposure-based techniques~\cite{eilertsen2017hdr,marnerides2018expandnet,marnerides2021deep}. 

More recent work has explored adapting diffusion models for HDR synthesis. Bracket Diffusion~\cite{bemana2025bracket} synchronizes multiple LDR exposure brackets from pretrained LDR diffusion models to reconstruct HDR images without retraining. LEDiff~\cite{wang2025lediff} introduces latent exposure fusion and trains HDR-specific components to recover highlight and shadow details. X2HDR~\cite{wu2026x2hdr} demonstrates that pretrained VAEs can encode HDR signals in perceptually uniform domains such as PQ or PU21~\cite{itur2025bt2100,azimi2021pu21}, but still requires fine-tuning the diffusion backbone, e.g., via LoRA~\cite{hu2022lora}, to mitigate SDR bias. Other recent efforts further explore diffusion-based LDR-to-HDR reconstruction for videos~\cite{yu2026diffhdr}.

In contrast, our approach does not modify model parameters. Instead, we directly control the output luminance distribution of a pretrained diffusion model at sampling time, enabling HDR-consistent generation without retraining.

\subsubsection{Classifier Guidance and Energy-Based Steering}

Controlling diffusion models at inference time has been widely studied through guidance techniques. Early methods introduce classifier guidance, where gradients from an external classifier are used to steer the sampling process toward desired classes~\cite{dhariwal2021diffusion}. Classifier-free guidance further simplifies this approach by jointly learning conditional and unconditional scores, enabling efficient trade-offs between fidelity and diversity without an external classifier~\cite{ho2022classifier}.

Beyond class-conditional control, recent work has generalized guidance to broader differentiable objectives. Score-based generative modeling provides a natural framework for incorporating conditional gradients into the reverse sampling process~\cite{song2020score}. Universal guidance methods allow diffusion models to be steered by user-defined guidance functions without retraining task-specific components~\cite{bansal2023universal}. Diffusion posterior sampling further applies gradient-based posterior guidance to general noisy inverse problems~\cite{chung2022diffusion}. Related training-free energy-guided approaches construct external energy functions from pretrained networks or attention mechanisms to control generated samples under diverse conditions~\cite{yu2023freedom,park2023energy}.

By defining an energy over luminance histograms, our method enables global, permutation-invariant control of output statistics. This allows us to manipulate HDR-relevant properties, such as dynamic range and highlight distribution, while leaving spatial structure primarily governed by the pretrained diffusion prior.

\subsection{Method Details \& Proofs}

\subsubsection{Soft Histogram Consistency}
\label{app:soft_hist} 

\begin{lemma}[Soft histogram bias]
\label{lem:softhist}
Let $h$ be the empirical $K$-bin histogram of $\{Y_i\}_{i=1}^{N}$ and $h_\sigma$ the soft histogram \eqref{eq:softhist}.  Assume bins are uniformly spaced with width $w=1/K$. Then
\begin{equation}
\|h_\sigma - h\|_{1}
 \;\le\; 2\sigma\sqrt{\tfrac{2}{\pi}}\cdot K \;+\; \mathcal{O}(K^{-1})
 \;=\; \mathcal{O}(\sigma K + K^{-1}).
\end{equation}
In particular, choosing \(\sigma=o(1/K)\) yields vanishing bias, while the practical choice \(\sigma=0.5/K\) keeps boundary leakage controlled and preserves differentiability of \(h_\sigma\) in \(Y\).
\end{lemma}
 
\begin{proof}
Per-bin error is bounded by the integral of a Gaussian kernel against the indicator of a width-$w$ bin minus a width-$w$ rectangle. Standard kernel-density bias analysis gives the leading term $2\sigma\sqrt{2/\pi}/w$ per bin and $\mathcal{O}(w)$ correction. Summing over $K$ bins yields the stated rate.
\end{proof}
 
This motivates the practical choice \(\sigma=0.5/K\) in Section~\ref{sec:experiments}: it keeps the soft histogram smooth while limiting boundary leakage.

\subsubsection{Spatial Decoupling} 
\label{app:geomprev}

\begin{corollary}[Equivariance of guided sampling]
\label{cor:geomprev}
Let $\Phi_\theta$ denote the unguided sampling map and let $\Phi_\theta^{s}$ denote its energy-guided counterpart with energy of the form in Proposition~\ref{prop:decoupling}. Then, for any rigid spatial transform $T\!\in\!O(2)\cap\mathrm{Stab}(\mathcal{X})$ that commutes with $v_\theta$, $T\circ\Phi_\theta^{s} = \Phi_\theta^{s}\circ T$. Consequently, all spatial symmetries of the diffusion prior are inherited by the guided sampler.
\end{corollary}
 
We empirically observe in Section~\ref{sec:experiments} that LumaGuide primarily alters luminance values while largely preserving where bright and dark regions occur. The pretrained diffusion prior remains the main source of geometry and semantics, while energy guidance adjusts global feature statistics along the prescribed feature axis.

\subsubsection{An Impossibility Result for Brightness Scaling}
\label{app:nogo}

A natural skeptical position is that LumaGuide merely brightens images, and that a global brightness scalar would suffice. We rule this out. For a distribution $P$ on $[0,1]$ with mean $\mu_P$ and positive standard deviation $\sigma_P$, write $\overline{P}$ for its \emph{shape}, the law of $(Y-\mu_P)/\sigma_P$, $Y\!\sim\!P$. Two distributions share the same shape iff one is an affine reparameterization of the other.

\begin{theorem}[No-go for affine luminance shaping]
\label{thm:nogo}
Let $P_0,P^{*}\!\in\!\mathcal{P}([0,1])$ have finite second moments and $\sigma_{P_0},\sigma_{P^{*}}\!>\!0$. For any clipped affine map $T(y)=\Pi_{[0,1]}(ay+b)$ with $a>0$,
\begin{equation}\label{eq:nogo}
\Wone(T_{\#}P_0,\,P^{*})
\;\ge\;
\sigma_{P^{*}}\,\Wone(\overline{P_0},\,\overline{P^{*}})\;-\;\rho,
\end{equation}
with equality and $\rho=0$ on the unclipped family, attained uniquely by the location\,--\,scale match
\[
T^{*}(y)\;=\;\mu_{P^{*}}+\tfrac{\sigma_{P^{*}}}{\sigma_{P_0}}(y-\mu_{P_0}).
\]
The clipping correction $\rho\!\ge\!0$ vanishes whenever $T^{*}([0,1])\!\subseteq\![0,1]$. In particular, when $\overline{P_0}\!\ne\!\overline{P^{*}}$ --- i.e.\ the SDR prior and the HDR target have different shapes --- the right-hand side is strictly positive, and \emph{no global affine brightness adjustment can generally match the target distribution}.
\end{theorem}

\begin{proof}

\emph{Affine maps preserve shape.} If $Y\!\sim\!P_0$, then $aY+b$ has mean $a\mu_{P_0}+b$ and standard deviation $a\sigma_{P_0}$, so its standardized version equals $(Y-\mu_{P_0})/\sigma_{P_0}$. Hence $\overline{T_{\#}P_0}=\overline{P_0}$ for every unclipped $T$: the family cannot leave the shape orbit of $P_0$.

\emph{$\Wone$ separates by shape.} The $1$-Wasserstein distance scales with affine reparameterizations, $\Wone(b+aX,\,c+aY)=a\,\Wone(X,Y)$ \cite{arjovsky2017wasserstein}. Combined with shape invariance, minimizing $\Wone(T_{\#}P_0,P^{*})$ over $(a,b)$ is a convex problem with unique solution $T^{*}$, and the minimum value equals $\sigma_{P^{*}}\Wone(\overline{P_0},\overline{P^{*}})$. This is strictly positive whenever the shapes differ, since $\Wone$ is a metric.

\emph{Clipping costs at most $\rho$.} The projection $\Pi_{[0,1]}$ is $1$-Lipschitz, so clipping a distribution can decrease its $\Wone$ to $P^{*}$ by at most the mass it pushes outside $[0,1]$. Defining $\rho$ as the infimum of this excess mass over $(a,b)$ gives \eqref{eq:nogo}. If $T^{*}([0,1])\!\subseteq\![0,1]$ no clipping is needed at $T^{*}$, hence $\rho=0$.\hfill$\square$
\end{proof}

\noindent Theorem~\ref{thm:nogo} explains the brightness-scaling baseline (Table~\ref{tab:ablation_bins}) at a fundamental level: $\alpha$-scaling is \emph{provably} sub-optimal whenever the SDR-biased prior and the HDR target differ in shape, as observed empirically.

\subsubsection{Idealized Guidance Schedule}
\label{app:schedule}

\begin{theorem}[Schedule under idealized noisy gradients]
\label{thm:schedule}
Consider the rectified flow interpolation $z_t=(1-t)z_0+t\epsilon$, $t\!\in\![0,1]$. Under (S1) bounded velocity error $\Var(v_\theta-v^{*})\!\le\!\sigma_v^2$ uniformly in $t$, (S2) trajectory sensitivity scales as $(1-t)^{2}$ in expected $\Wone$-influence on the terminal sample, and (S3) endpoint variance is regularized, an influence-weighted schedule takes the form
\begin{equation}\label{eq:scheduleopt}
s^{*}(t)\;\propto\;\frac{(1-t)^{2}\,t^{2}}{\bigl(t^{2}+(1-t)^{2}\bigr)^{2}} .
\end{equation}
This schedule peaks at $t=\tfrac{1}{2}$.
\end{theorem}

\begin{proof}
The schedule follows by combining three idealized factors: the trajectory influence $(1-t)^2$, the reconstruction-noise factor induced by the flow estimate, and an endpoint variance regularizer that prevents the schedule from concentrating at the boundaries. Under these assumptions, the resulting influence-to-noise weighting is proportional to
\[
\frac{(1-t)^2t^2}{\bigl(t^2+(1-t)^2\bigr)^2}.
\]
Differentiating this expression shows that its unique interior maximum occurs at \(t=\tfrac{1}{2}\).
\end{proof}

\begin{remark}[Why constant beats $s^{*}$ in practice]
\label{rem:constpractice}
Theorem~\ref{thm:schedule} is derived under idealized assumptions on trajectory sensitivity and noise; in practice, mid-trajectory dynamics are dominated by the base velocity field $v_\theta$, suppressing the effective marginal benefit of guidance there. Conversely, early-time guidance enables low-cost coarse global adjustments and late-time guidance enables high-luminance refinement.  These two edge regimes are penalized by $s^{*}$ and rewarded by the constant schedule, which explains the empirical hierarchy
$\text{constant} \succ s^{*} \succ \text{windowed}$ observed in Table~\ref{tab:1024_results}.
\end{remark}

\subsubsection{Temporal Stability under TLC}
 
\begin{proposition}[Bounded inter-frame luminance drift]
\label{prop:tlc}
Let $Y_\tau, Y_{\tau-1}$ be the PQ-luminance maps of adjacent LumaGuide-decoded video frames with frame-wise velocities $v_\theta^{(\tau)}$ and $v_\theta^{(\tau-1)}$. Suppose the inter-frame velocity discrepancy on the high-luminance mask satisfies $\|v_\theta^{(\tau)} - v_\theta^{(\tau-1)}\|_{\Omega_\tau} \le \kappa$. Then with the TLC term of Definition~\ref{def:tlc},
\begin{equation}
\E\bigl\|Y_\tau - Y_{\tau-1}\bigr\|_{\Omega_\tau}
\;\le\;
\frac{\kappa}{\lambda_{\mathrm{TLC}}}
\;+\;
\mathcal{O}\!\left(\frac{\sigma_v^{2}}{\lambda_{\mathrm{TLC}}}\right).
\end{equation}
\end{proposition}
  
The TLC-augmented energy is convex in $Y_\tau$ for fixed $Y_{\tau-1}$ and $h_\sigma$. The first-order optimality on the masked region gives $\nabla_{Y_\tau}(\Wone\!+\!\lambda_{\mathrm{TLC}}\mathcal{L}_{\mathrm{TLC}})=0$, i.e.\ $2\lambda_{\mathrm{TLC}}(Y_\tau-Y_{\tau-1})\odot\Omega_\tau = -\nabla\Wone$. Taking norms and using $\|\nabla\Wone\|\le \kappa+\|\delta\|$,
\begin{equation*}
\|Y_\tau-Y_{\tau-1}\|_{\Omega_\tau}\le\frac{\kappa+\|\delta\|}{2\lambda_{\mathrm{TLC}}}.
\end{equation*}
Taking expectations and using $\E\|\delta\|^{2}\le\sigma_v^{2}$ yields the claim. $\square$

\subsubsection{Proof of Theorem~\ref{thm:descent}}
\label{app:descent}
 
Let $\Phi(z):=E(\Dec(z-t\,v_\theta(z,t)))$ for fixed $t$, so that
$\Phi$ is the energy as a function of the latent.  By assumptions
(A1)--(A2) and chain rule, $\Phi$ is $L$-smooth with
$L=L_E\,L_\Dec^{2}\,(1+t\,\|J_{v_\theta}\|)^{2}$. On a compact
$t$-interval bounded away from $1$, $L$ is bounded; absorb the constant
into the symbol.

 Since the sampler steps from \(t_n\) to \(t_{n+1}<t_n\), let \(\Delta t_n:=t_n-t_{n+1}>0\). The guided update can be written as
\[
\Delta z
=
-\Delta t_n\bigl(v_\theta+s\nabla\Phi\bigr).
\]
By the descent lemma for \(L\)-smooth functions,
\begin{equation*}
\Phi(z+\Delta z)
\le
\Phi(z)
+
\langle\nabla\Phi,\Delta z\rangle
+
\tfrac{L}{2}\|\Delta z\|^{2}.
\end{equation*}

Let \(\eta_t:=s\Delta t_n\). Substituting $\Delta z$ and taking expectation over the velocity-error
$\delta=v_\theta-v^{*}$,
\begin{align*}
\E[\Phi(z+\Delta z)]
&\le \Phi(z) - \eta_t\,(1-\tfrac{L\eta_t}{2})\E\|\nabla\Phi\|^{2}\\
&\quad + \E\langle\nabla\Phi,v^{*}\rangle\Delta t_n
+ L\Delta t_n^{2}\E\|\delta\|^{2}/2.
\end{align*}
The middle term is the unguided drift; under (A3) this term is bounded.
We absorb $L\Delta t_n/2 \cdot \E\|\delta\|^{2}$ into the noise term. The factor
$t^{2}$ enters because, in flow-matching coordinates, $\xhat$ amplifies
$\delta$ by $t/(1-t)$, contributing $t^{2}/(1-t)^{2}$ which we bound
on $t\!\in\![0,1-\varepsilon]$ by $t^{2}/\varepsilon^{2}$ and absorb
$1/\varepsilon^{2}$ into the constant $C$. The condition \(\eta_t=s\Delta t_n\le 2/L\) guarantees the factor \(1-L\eta_t/2\) is non-negative.
$\square$

\subsubsection{Proof of Proposition~\ref{prop:decoupling}}
\label{app:decoupling}

\begin{enumerate}
\item[\textnormal{(i)}]\textbf{Permutation equivariance.} 
\item[\textnormal{(ii)}]\textbf{Pixel-wise gradient form.} 
\begin{equation}
\bigl[\nabla_x E(x)\bigr]_{\cdot,i}
\;=\; \nabla_{x_{\cdot,i}} \varphi(x_{\cdot,i}) \cdot
       g\!\bigl(\varphi(x_{\cdot,i});\,h_\sigma(x),h^{*}\bigr),
\end{equation}
\item[\textnormal{(iii)}]\textbf{Spatial information bound.}
\end{enumerate}

\begin{proof}
(i) Since \(h_\sigma(x)\) is obtained by summing identical per-pixel soft assignments over all spatial locations, it is invariant to any spatial permutation \(\pi\in S_N\). Hence
\[
h_\sigma(\pi\cdot x)=h_\sigma(x),
\qquad
E(\pi\cdot x)=E(x).
\]
Differentiating this identity gives
\[
\nabla_xE(\pi\cdot x)
=
\pi\cdot\nabla_xE(x).
\]

(ii) Applying the chain rule to \(f(x)_i=\varphi(x_{\cdot,i})\) gives
\[
\bigl[\nabla_x E(x)\bigr]_{\cdot,i}
=
g\!\left(\varphi(x_{\cdot,i});h_\sigma(x),h^*\right)
\nabla_{x_{\cdot,i}}\varphi(x_{\cdot,i}),
\]
for some scalar function \(g\). Thus, pixels interact through the global histogram statistics, but not through their spatial coordinates.

(iii) Since \(E\) depends on \(x\) only through the multiset
\[
\{f(x)_i\}_{i=1}^{N},
\]
it cannot distinguish images whose feature values differ only by a spatial permutation.
\end{proof}
 
\subsubsection{Proof of Proposition~\ref{prop:pqstab}}
\label{app:pqstab}

 Let $L\!\in\![L_{\min},L_{\max}]$ denote scene luminance in nits with
$L_{\max}/L_{\min}\!\ge\!10^4$ (HDR range).  For a fixed bin grid in PQ space,
\begin{enumerate}
\item[\textnormal{(i)}] $\PQ$ is monotone, $C^\infty$, and $\PQ'(L)\,L$ is bounded above and below by positive constants on $[L_{\min},L_{\max}]$. Hence per-pixel gradients $\partial E/\partial L$ converted to PQ space are scale-balanced across the dynamic range.
\item[\textnormal{(ii)}] In linear space, the inverse PQ Jacobian induces polynomially mismatched scales across luminance ranges; for large \(L\), the corresponding factor scales as \(\Theta(L^{1-1/m_2})\), yielding gradient variance that grows as \(\Omega(L_{\max}^{2(1-1/m_2)})\).
\item[\textnormal{(iii)}] Consequently, the noise-to-signal ratio of the histogram-shaping gradient in PQ space is bounded by a constant independent of $L_{\max}$, while in linear space it grows polynomially with the dynamic range.
\end{enumerate}

\begin{proof} 
The PQ OETF is
\begin{equation*}
\PQ(L) \;=\;
\Bigl(\tfrac{c_1+c_2(L/L_p)^{m_1}}{1+c_3(L/L_p)^{m_1}}\Bigr)^{m_2},
\qquad L_p = 10{,}000\;\text{nits},
\end{equation*}
with constants $m_1\!=\!2610/16384$, $m_2\!=\!2523/4096$,
$c_1\!=\!3424/4096$, $c_2\!=\!2413/128$, $c_3\!=\!2392/128$.
Differentiating,
$\PQ'(L) = m_2 \,\PQ(L)^{1-1/m_2}\!\cdot\! \frac{(c_2-c_1c_3)\,m_1\,(L/L_p)^{m_1-1}/L_p}{(1+c_3(L/L_p)^{m_1})^{2}}$.
For $L\!\gg\!1$ nit, the leading scaling is $L^{m_1-1}$. Multiplying
by $L$ to convert to the relative scale gives $L^{m_1}\!\cdot\!\text{const}$
which, since $m_1\!\approx\!0.16$, is bounded above and below by
constants on any range $[L_{\min},L_{\max}]$ with
$\log_{10}(L_{\max}/L_{\min})\!\le\!4$. This is statement (i).

Statement (ii) follows by considering the inverse mapping from PQ perturbations to linear luminance perturbations. Since the inverse PQ Jacobian scales as \(\Theta(L^{1-1/m_2})\) for large \(L\), approximately uniform perturbations in PQ space induce polynomially mismatched luminance-scale gradients. Squaring this factor over the HDR range yields the stated variance growth.
 
Statement (iii) follows by combining (i) and the boundedness of the
soft histogram Jacobian (whose $\ell^{\infty}$ norm is at most
$1/(\sigma\sqrt{2\pi})\!\cdot\!1/N$ per pixel-bin pair). $\square$
\end{proof}

\subsection{Using Pretrained VAE in Perceptual Uniform Space}

A key observation underlying our method is that pretrained LDR VAEs can reconstruct HDR signals with high fidelity when the inputs are expressed in a perceptually uniform space (e.g., PQ or PU21), without any retraining.

The main issue is not model capacity, but \textbf{representation mismatch}. Linear HDR is distributed very differently from LDR data: it is heavy-tailed, dominated by extreme highlights, and poorly aligned with human perceptual sensitivity. As a result, directly feeding linear HDR into an LDR-trained VAE leads to distorted latent representations and degraded reconstructions.

Perceptually uniform encodings (PQ/PU21) address this mismatch by compressing highlights and redistributing precision toward low and middle luminance regions, producing statistics that are significantly closer to LDR image distributions. Empirically, prior work~\cite{wu2026x2hdr} shows that PQ/PU21-encoded HDR can be reconstructed by pretrained VAEs with quality comparable to LDR inputs, while linear HDR fails to do so.

From a latent-space perspective, PQ encoding brings HDR signals closer to the statistics seen by LDR-trained VAEs, enabling more stable encoding and decoding without modifying the VAE.

This observation implies that HDR generation is primarily a \textbf{distribution alignment problem}. Once HDR signals are represented in a perceptually aligned space, pretrained generative models can be directly reused. Our method builds on this insight and further performs distribution-level control in PQ space, without requiring any model adaptation.

\subsection{Additional Experimental Results}
\label{app:exps}

\begin{figure}[h]
    \centering
    \includegraphics[width=\linewidth]{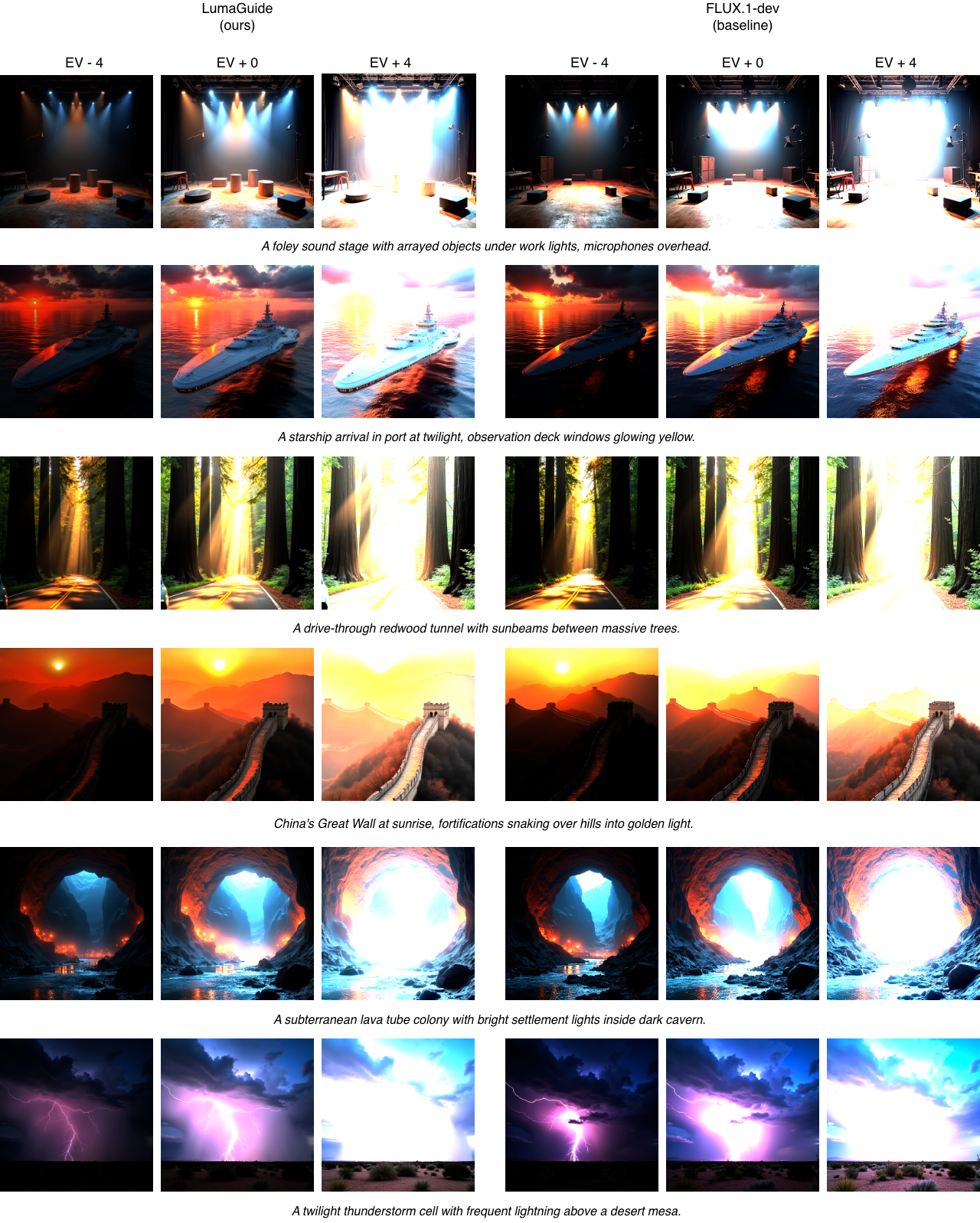}
    \caption{Additional qualitative results across diverse prompts and scenes.}
    \label{fig:appendix_more_images}
\end{figure}

We provide additional qualitative results across a wider range of prompts and scenes to demonstrate the robustness of \textbf{LumaGuide}. Figure~\ref{fig:appendix_more_images} shows diverse scenarios, including indoor scenes, night environments, and high-contrast outdoor settings. Across all cases, \textbf{LumaGuide} consistently reshapes luminance distributions toward HDR-like profiles while preserving semantic content. Compared to the baseline, our method better maintains characteristic HDR properties, including stronger contrast between highlights and shadows, more localized high-intensity regions (e.g., light sources, reflections, and sky highlights), and reduced overexposure in surrounding areas. In particular, bright regions remain more structured and spatially confined, rather than spreading into large saturated areas, resulting in more realistic highlight rendering.

We further evaluate the method at higher resolution using $1024 \times 1024$ generation with Flux.1-dev~\cite{flux}. Qualitative comparisons between $512 \times 512$ and $1024 \times 1024$ are shown in Figure~\ref{fig:resolution_comparison}. The higher-resolution results exhibit improved texture detail and fine-grained structure, while maintaining consistent luminance distribution shaping. Notably, HDR characteristics such as highlight sharpness, contrast separation, and luminance range are preserved across resolutions, indicating that the proposed guidance operates consistently at different spatial scales.

Quantitative comparisons are summarized in Table~\ref{tab:1024_results}, showing consistent improvements in perceptual quality metrics at higher resolution. These results suggest that \textbf{LumaGuide} generalizes across resolutions and is not restricted to a specific image scale, although the optimal guidance strength may need to be adjusted accordingly.

\begin{table}[h]
\centering
\small
\caption{
Comparison between \(512\times512\) and \(1024\times1024\) generation using \textbf{LumaGuide}. Higher resolution improves perceptual quality, alignment, and dynamic range, at the cost of increased runtime.}
\label{tab:1024_results}
\begin{tabular}{lcccc}
\toprule
\textbf{Method} & $\mathrm{DR}_{\text{stops}}$ $\uparrow$ & Q-quality $\uparrow$ & Q-align $\uparrow$ & Time $\downarrow$ \\
\midrule
\textbf{\textbf{LumaGuide} @ $512^2$}  & 14.99 & 0.568 & 0.814 & \textbf{7.81~s} \\
\textbf{\textbf{LumaGuide} @ $1024^2$} & \textbf{15.95} & \textbf{0.584} & \textbf{0.833} & 32.3~s \\
\bottomrule
\end{tabular}
\end{table}

\begin{figure}
    \centering
    \includegraphics[width=\linewidth]{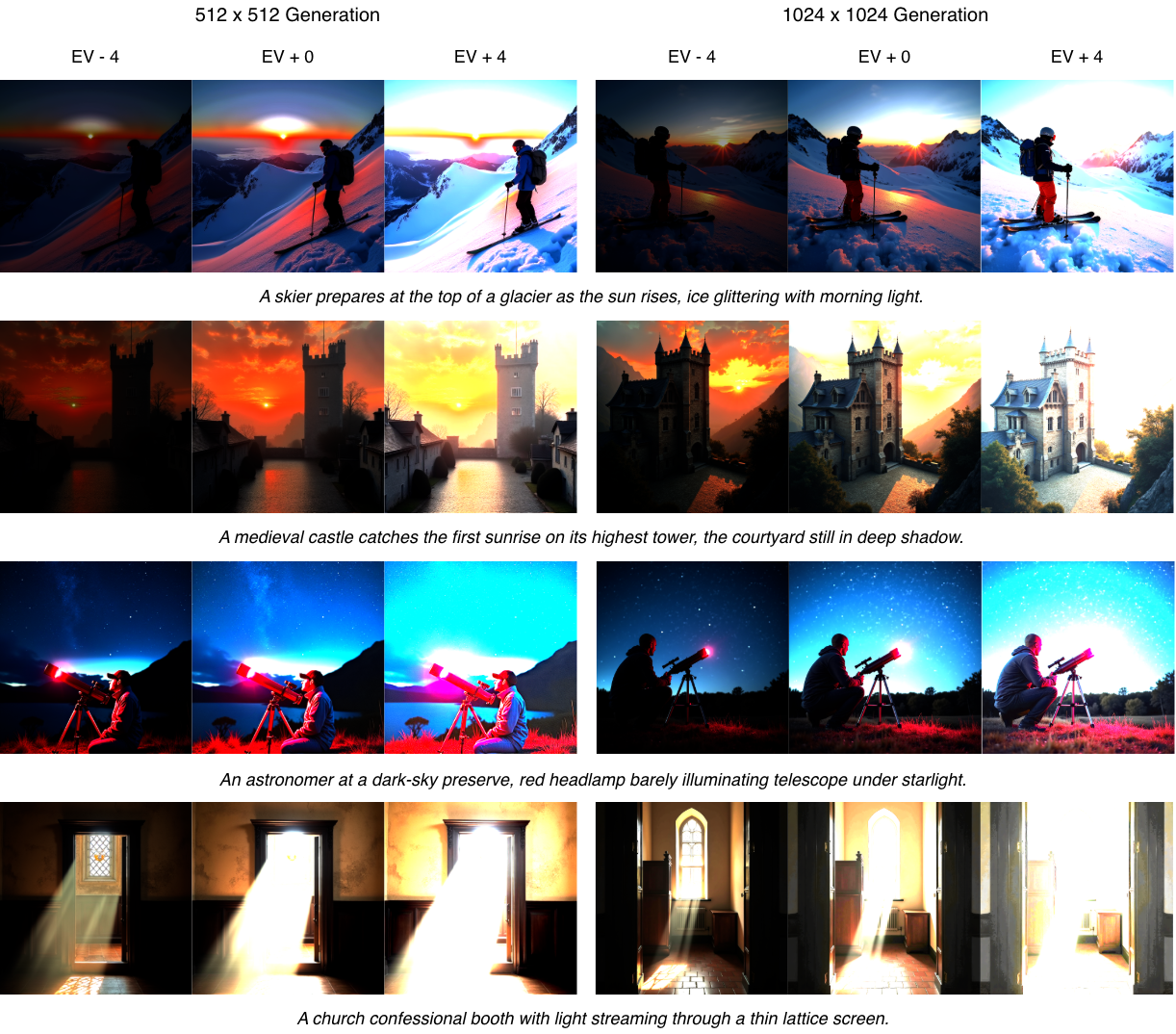}
    \caption{
    Resolution comparison between $512 \times 512$ and $1024 \times 1024$ generation. Higher-resolution results exhibit improved texture detail and spatial coherence while preserving consistent luminance distribution shaping. 
    }
    \label{fig:resolution_comparison}
\end{figure}

\subsection{Additional Ablation Studies}
\label{app:ablation}

\subsubsection{Number of histogram bins.}

\begin{table}[t]
\centering
\small
\caption{
Effect of histogram bin count $K$ on distribution alignment and perceptual quality. Each configuration uses a separately trained regressor to generate $K$-bin target distributions.
}
\label{tab:ablation_bins}
\begin{tabular}{lcccccc}
\toprule
$K$ & uW1 $\downarrow$ & p50$_{\text{dist}}$ $\downarrow$ & p99$_{\text{dist}}$ $\downarrow$ & $\mathrm{DR}_{\text{stops}}$ $\uparrow$ & Q-quality $\uparrow$ & Q-alignment $\uparrow$ \\
\midrule
16  & 0.696 & 0.056 & 0.113 & \textbf{15.48} & 0.554 & \textbf{0.823} \\
\rowcolor{green!12}
32  & \textbf{0.576} & 0.024 & 0.053 & 14.99 & \textbf{0.568} & 0.814 \\
64  & 0.908 & \textbf{0.015} & \textbf{0.033} & 14.54 & 0.533 & 0.766 \\
128 & 2.317 & 0.017 & 0.037 & 14.52 & 0.490 & 0.684 \\
256 & 4.734 & 0.019 & 0.039 & 14.26 & 0.371 & 0.470 \\
\bottomrule
\end{tabular}
\end{table}

We study the impact of histogram resolution by varying the number of bins $K$ used to represent luminance distributions. For each value of $K$, we retrain the text-to-histogram regressor to predict $K$-dimensional PQ histograms, ensuring consistency with the binning scheme. Results are reported in Table~\ref{tab:ablation_bins}.

We observe that increasing \(K\) improves fine-grained distribution matching up to \(K=64\), as reflected by lower percentile errors \(p50_{\text{dist}}\) and \(p99_{\text{dist}}\). However, the overall distribution distance (uW1) and perceptual quality degrade at larger \(K\). This indicates that overly fine histogram resolution leads to diminishing returns in distribution alignment and may destabilize the optimization. We attribute this behavior to increased sensitivity to small bin-level discrepancies. As \(K\) increases, the guidance objective imposes stronger and more localized constraints, producing higher-variance gradients that can over-correct luminance values and degrade spatial coherence. Additionally, predicting high-dimensional target histograms introduces additional difficulty for the regressor, which may further contribute to performance degradation.

Overall, $K=32$ provides a favorable trade-off between controllability and stability. It achieves strong perceptual quality and semantic alignment while maintaining competitive distribution accuracy, suggesting that intermediate histogram resolution is sufficient for capturing HDR luminance characteristics without over-constraining the optimization.

\subsubsection{Guidance scale}
\label{app:guidance_scale}
\begin{figure}[t]
\centering
\includegraphics[width=\linewidth]{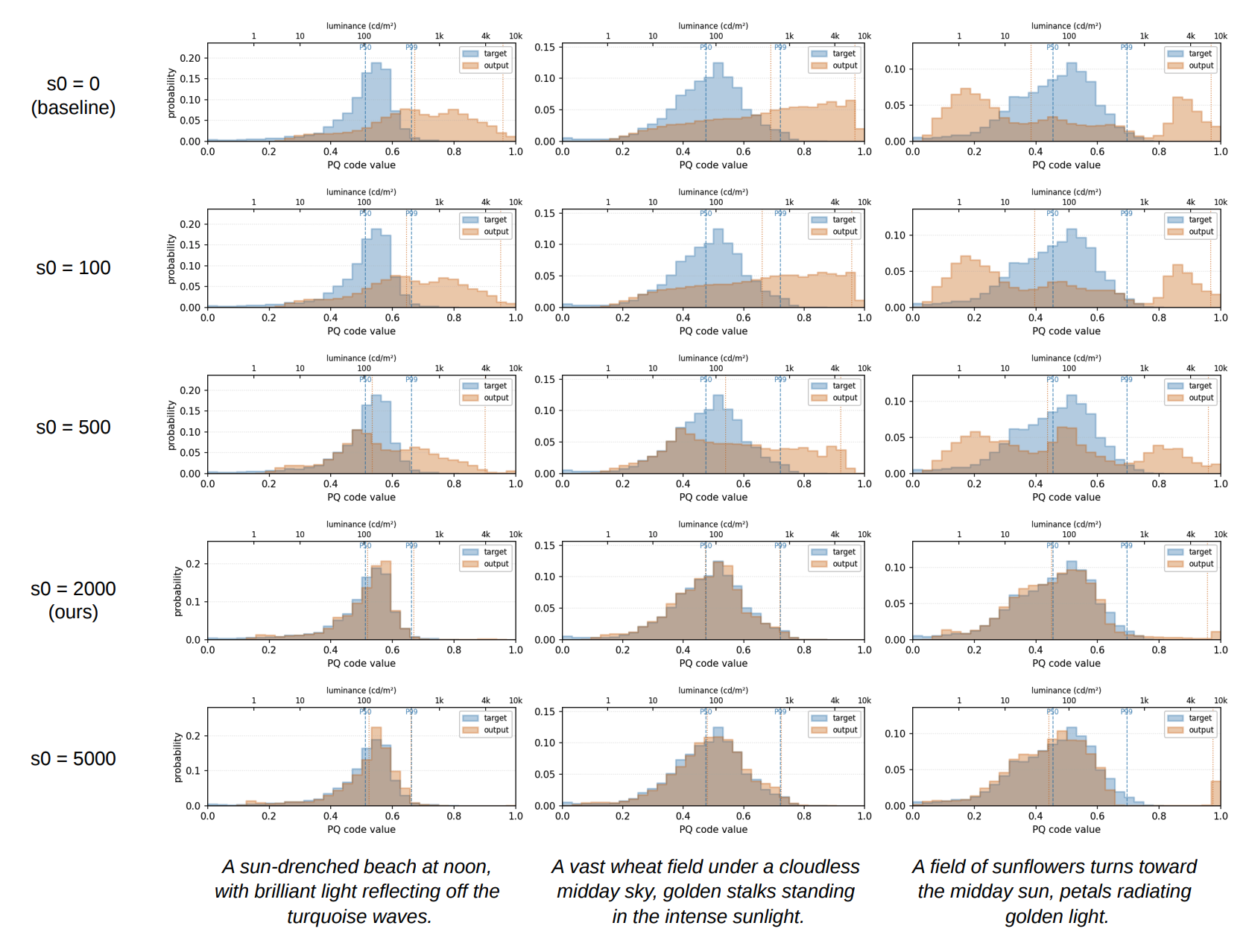}
\caption{Luminance histogram alignment in PQ space. \textbf{LumaGuide} reshapes the output distribution toward the target HDR profile.}
\label{fig:hist_alignment}
\end{figure}

\begin{figure}[t]
\centering
\includegraphics[width=\linewidth]{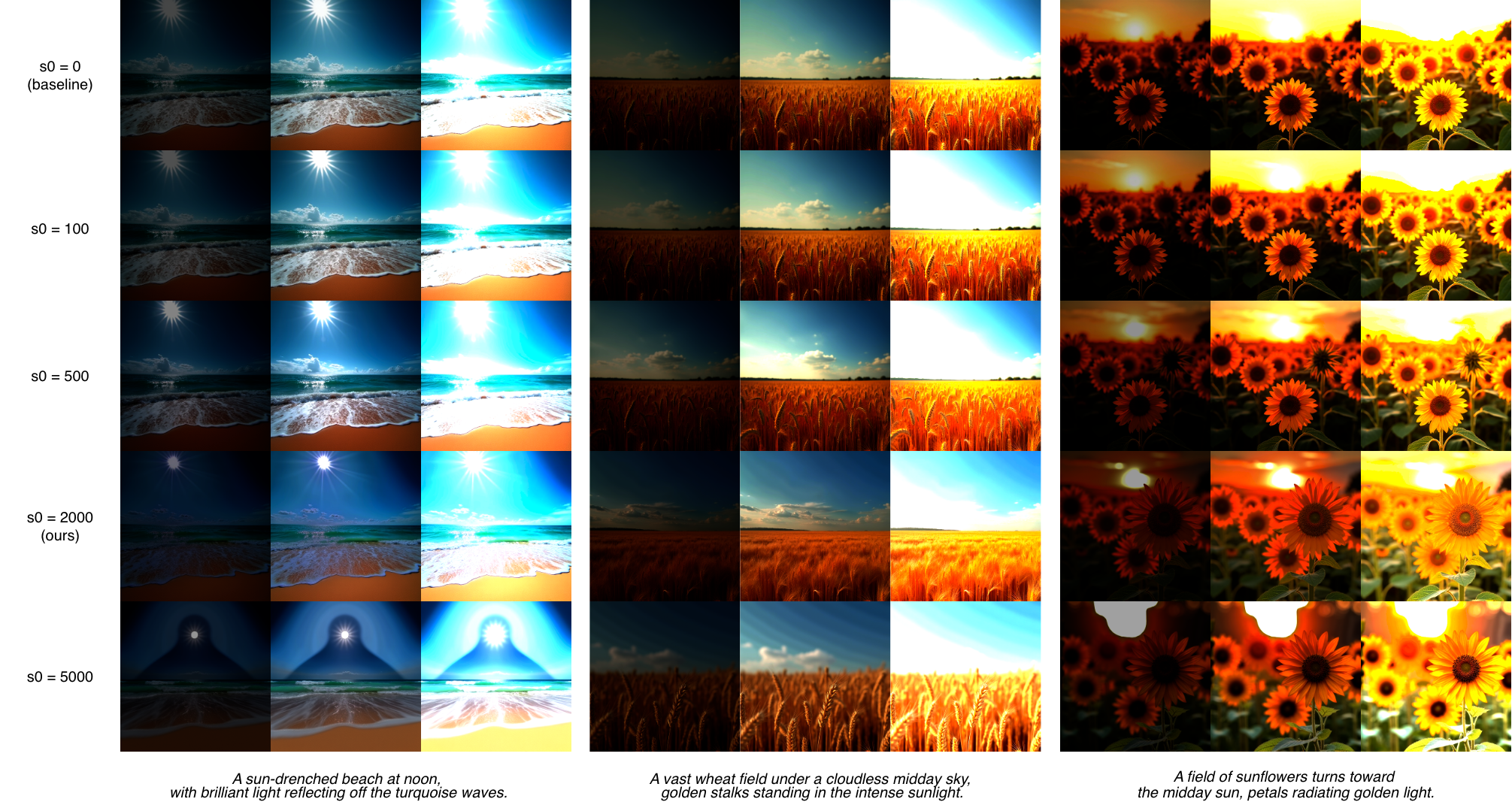}
\caption{Effect of guidance strength $s_0$ on luminance distribution and visual appearance. Each row corresponds to a different guidance scale, and columns show representative scenes under increasing exposure levels.}
\label{fig:guidance_visualization}
\end{figure}

We provide a more detailed analysis of the effect of guidance strength $s_0$ on distribution alignment and visual quality. While the main paper reports the overall trend, here we examine how different regimes of $s_0$ influence luminance statistics and perceptual behavior.

As shown in Table~\ref{tab:scale}, increasing \(s_0\) generally reduces distribution discrepancy up to an intermediate range, as measured by the Wasserstein-1 (\(W_1\)) distance. This confirms that the guidance signal effectively steers the generated samples toward the target luminance distribution. Improvements are observed across both global (uW1) and percentile-based metrics ($p50_{\text{dist}}$, $p99_{\text{dist}}$), indicating that the alignment affects the full distribution rather than a specific range.

At moderate values (e.g., $s_0 \in [500, 2000]$), the constraint is strong enough to redistribute luminance mass from extreme highlights toward mid-tones, improving perceptual balance and reducing overexposure. At higher values of $s_0$, however, the constraint becomes overly dominant relative to the diffusion prior. As a result, the optimization becomes increasingly sensitive to small distribution discrepancies, leading to over-correction and degradation of spatial coherence. 

Figure~\ref{fig:guidance_visualization} further illustrates this trade-off. At low $s_0$, outputs remain close to the baseline distribution and exhibit limited HDR characteristics. As $s_0$ increases, the outputs progressively align with the target distribution, improving highlight structure and shadow detail. Beyond a certain point, further increases in $s_0$ degrade perceptual quality due to over-constrained optimization. Based on this trade-off, we select \(s_0=2000\) as the default operating point, since it achieves strong alignment while preserving perceptual quality and avoiding the over-constrained behavior observed at larger scales.

\subsubsection{Guidance schedule}

\begin{table}[t]
\centering
\small
\caption{
Ablation of guidance schedules under integral normalization.
}
\label{tab:ablation_schedule}
\begin{tabular}{lccccccc}
\toprule
Schedule & Form & uW1 $\downarrow$ & p50$_{\text{dist}}$ $\downarrow$ & p99$_{\text{dist}}$ $\downarrow$ & DR $\uparrow$ & Q-quality $\uparrow$ & Q-align $\uparrow$ \\
\midrule
\rowcolor{green!12}
\textbf{constant} 
& $s(t)=s_0$ 
& \textbf{0.576} 
& \textbf{0.024} 
& \textbf{0.053} 
& \textbf{14.99}
& \textbf{0.568} 
& \textbf{0.814} \\

snr\_weighted 
& bell-shaped 
& 0.979 
& 0.037 
& 0.079 
& 14.94
& 0.538 
& 0.810 \\

late\_only 
& windowed 
& 1.484 
& 0.049 
& 0.105 
& 14.96 
& 0.514 
& 0.787 \\

\bottomrule
\end{tabular}
\end{table}
We analyze the role of the guidance schedule $s(t)$ in distribution shaping. Consider the rectified flow interpolation
\[
z_t = (1-t) z_0 + t \epsilon, \qquad t \in [0,1],
\]
with signal-to-noise ratio $\mathrm{SNR}(t) = \frac{(1-t)^2}{t^2}$. The guided update is
\[
v_{\text{guided}}(z_t, t) = v_\theta(z_t, t) + s(t)\,\nabla_{z_t} E.
\]

We first consider an idealized schedule motivated by expected downstream energy reduction. Two factors govern the effectiveness of guidance at time $t$:

\textbf{(i) Gradient reliability.}
Let $\hat z_0 = z_t - t\,v_\theta(z_t, t)$. Writing $v_\theta = v + \delta$ with prediction error $\delta$, we have
\[
\hat z_0 = z_0 - t\,\delta.
\]
Assuming $\mathrm{Var}(\delta) \propto 1$, the variance of the reconstruction scales as $\mathrm{Var}(\hat z_0) \propto t^2$, and hence
\[
\mathrm{Var}(\nabla E) \propto t^2.
\]

\textbf{(ii) Trajectory sensitivity.}
A perturbation at time $t$ propagates to the final sample with magnitude proportional to the remaining trajectory length. In rectified flow, this scales as $(1-t)$, giving an effective sensitivity $(1-t)^2$.

Combining these two factors suggests the scaling
\[
s^*(t) \;\propto\; \frac{(1-t)^2}{t^2}.
\]
Introducing a variance floor and normalizing leads to the bounded form
\[
s^*(t) \;\propto\; \frac{(1-t)^2 t^2}{\big(t^2 + (1-t)^2\big)^2},
\]
which peaks at $t = 0.5$.

\vspace{0.5em}
We compare this SNR-weighted schedule with a constant schedule and a time-windowed variant. All schedules are normalized to have equal $\int s(t)\,dt$. Results are shown in Table~\ref{tab:ablation_schedule}.

Empirically, the constant schedule outperforms all alternatives across both distributional and perceptual metrics. The SNR-weighted schedule yields nearly $2\times$ higher uW1, and the windowed schedule performs worst overall.

The discrepancy arises because the assumptions above do not hold uniformly across the trajectory. In practice, mid-range timesteps are dominated by the base flow dynamics, reducing the effective impact of guidance. Conversely, early timesteps allow low-cost global adjustments, while late timesteps remain sensitive for high-luminance refinement. As a result, concentrating guidance near $t=0.5$ is suboptimal.

A constant schedule distributes guidance uniformly across the trajectory, enabling both global and local adjustments, and yields the best overall performance.

\subsubsection{Effect of HDR encoding (PQ vs. PU21)}

\begin{table}[t]
\centering
\small
\caption{
Comparison of PQ and PU21 encoding for distribution shaping at different resolutions.
}
\label{tab:encoding_ablation}
\begin{tabular}{lcccc}
\toprule
Encoding & Res & DR (stops) $\uparrow$ & Q-quality $\uparrow$ & Q-alignment $\uparrow$ \\
\midrule
\rowcolor{green!12}
PQ   & 512  & 14.99 & \textbf{0.568} & \textbf{0.814} \\
PU21 & 512  & \textbf{15.95} & 0.556 & 0.813 \\
\midrule
\rowcolor{green!12}
PQ   & 1024 & 16.02 & 0.584 & 0.833 \\
PU21 & 1024 & \textbf{16.40} & \textbf{0.588} & \textbf{0.834} \\
\bottomrule
\end{tabular}
\end{table}

We compare two perceptually motivated HDR encoding schemes, PQ and PU21, for distribution shaping. Results are summarized in Table~\ref{tab:encoding_ablation}.

PU21 consistently achieves higher dynamic range across resolutions, indicating stronger expansion in extreme luminance regions. It also provides slightly better perceptual quality and alignment at higher resolution. However, PQ remains competitive across all settings and achieves stronger perceptual quality at lower resolution. More importantly, PQ is a widely adopted and standardized HDR encoding, making it more suitable for general-purpose use.

Overall, both encodings are effective for distribution shaping. We adopt PQ as the default due to its robustness and broader compatibility, while noting that PU21 may provide advantages in extreme dynamic range scenarios.

\subsection{Additional Baselines}
\label{app:baselines}
A key question for distribution shaping methods is whether the observed behavior is intrinsic to the method or tied to a specific backbone. To establish that \textbf{LumaGuide} is architecture-agnostic, we evaluate it across three representative diffusion backbones: Flux.1-dev, SD3-medium, and SDXL base 1.0. Their key characteristics are summarized in Table~\ref{tab:backbone_comparison}. These models differ along multiple axes, including architecture (MMDiT vs.\ UNet), scheduler family (rectified flow vs.\ $\epsilon$-prediction), VAE design (16-channel native vs.\ 4-channel sRGB-trained), and model scale. This selection allows us to assess whether \(W_1\)-guided distribution shaping remains effective across substantially different model designs.

The three backbones are chosen to span the design space of modern diffusion models. Flux.1-dev serves as the strongest open-source baseline and our default backbone. SD3-medium shares the same rectified-flow formulation but differs in scale and training distribution, allowing us to isolate whether the method depends on model size or data. SDXL, in contrast, represents a previous-generation UNet-based model with $\epsilon$-prediction and an sRGB-trained VAE, providing a maximal shift in both architecture and latent representation. Together, these models provide a stress test of how scheduler family, architecture, and VAE design affect distribution shaping behavior.

\begin{table}[t]
\centering
\small
\setlength{\tabcolsep}{10pt}
\renewcommand{\arraystretch}{1.2}
\caption{Comparison of backbone characteristics used for cross-backbone evaluation. The three models span different architectures, scheduler families, and VAE designs, enabling controlled analysis of distribution shaping behavior.}
\label{tab:backbone_comparison}
\begin{tabular}{@{}l ccc@{}}
\toprule
\textbf{Property} & \textbf{Flux.1-dev} & \textbf{SDXL} & \textbf{SD3-medium} \\
\midrule
Backbone        & \cellcolor{accent} MMDiT          & UNet                  & \cellcolor{accent} MMDiT \\
Scheduler       & \cellcolor{accent} Rectified flow & $\epsilon$-prediction & \cellcolor{accent} Rectified flow \\
Parameters      & 12B                               & 3.5B                  & 2B \\
Text encoder    & \cellcolor{accent} T5 + CLIP      & CLIP (dual)           & \cellcolor{accent} T5 + CLIP \\
VAE channels    & \cellcolor{accent} 16             & 4                     & \cellcolor{accent} 16 \\
\bottomrule
\end{tabular}
\end{table}

\begin{table}[t]
\centering
\small
\caption{
Cross-backbone comparison of \textbf{LumaGuide} at the selected operating points for each model. Metrics report perceptual quality, alignment, and dynamic range.
}
\label{tab:all_baseline_results}
\begin{tabular}{lccc}
\toprule
\textbf{Backbone} & Q-quality $\uparrow$ & Q-align $\uparrow$ & DR (stops) $\uparrow$ \\
\midrule
\textbf{Flux (ours)} & \textbf{0.568} & \textbf{0.814} & 14.99 \\
SD3                 & 0.512 & 0.795 & 15.42 \\
SDXL                & 0.431 & 0.655 & \textbf{15.90} \\
\bottomrule
\end{tabular}
\end{table}

\subsubsection{Results on SD3.}
\begin{figure}
    \centering
    \includegraphics[width=\linewidth]{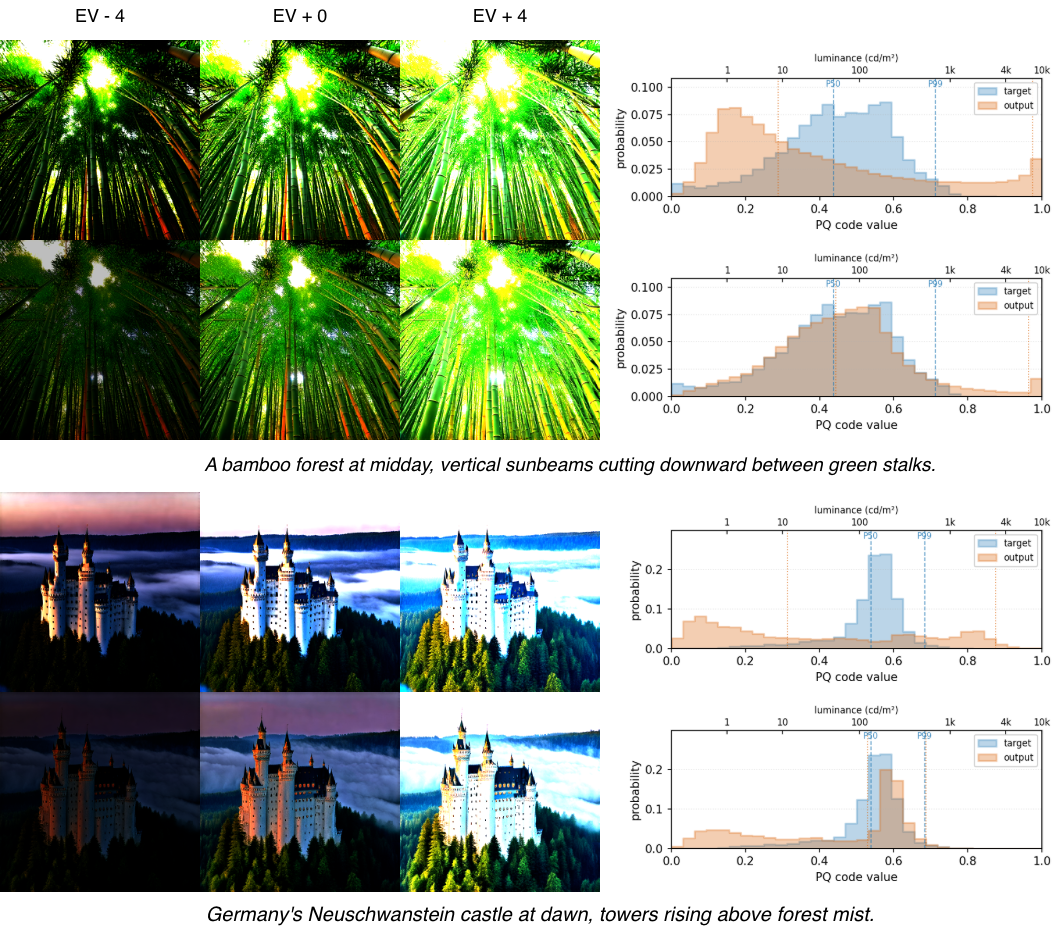}
    \caption{Qualitative results on SD3 and luminance distribution matching at \(s_0=1000\).}
    \label{fig:sd3_qual}
\end{figure}

\begin{table}[t]
\centering
\small
\caption{
Effect of guidance scale $s_0$ on distribution alignment for SD3-medium.
}
\label{tab:sd3_s0_sweep}
\begin{tabular}{lccccc}
\toprule
$s_0$ & uW1 $\downarrow$ & p50$_{\text{dist}}$ $\downarrow$ & p99$_{\text{dist}}$ $\downarrow$ & nits$_{50}$ & nits$_{99}$ \\
\midrule
0    & 4.358 & 0.130 & 0.251 & 83.03 & 6708.69 \\
100  & 3.755 & 0.110 & 0.246 & 62.83 & 6657.08 \\
500  & 2.121 & 0.066 & 0.199 & 36.80 & 5298.82 \\
\rowcolor{green!12}
1000 & 1.145 & 0.038 & 0.149 & 37.31 & 4292.97 \\
2000 & 0.762 & 0.028 & 0.120 & 40.84 & 3944.31 \\
\bottomrule
\end{tabular}
\end{table}

We first evaluate \textbf{LumaGuide} on SD3-medium. Quantitative results are summarized in Table~\ref{tab:all_baseline_results}, and qualitative comparisons are shown in Figure~\ref{fig:sd3_qual}. Similar to Flux.1, increasing the guidance scale consistently improves distribution alignment on SD3, as shown in Table~\ref{tab:sd3_s0_sweep}. The method consistently reduces distribution error (uW1 and percentile distances). Qualitatively, the generated images exhibit improved highlight structure and more consistent HDR-specified luminance distributions, as evidenced by the closer alignment between output and target histograms.

Consistent with observations on Flux.1, overly strong guidance leads to visible artifacts, such as banding and unnatural textures. We therefore select \(s_0=1000\) as a balanced operating point for SD3, since it achieves strong distribution alignment while avoiding the artifacts observed at larger scales.

\subsubsection{Results on SDXL.}

\begin{figure}
    \centering
    \includegraphics[width=\linewidth]{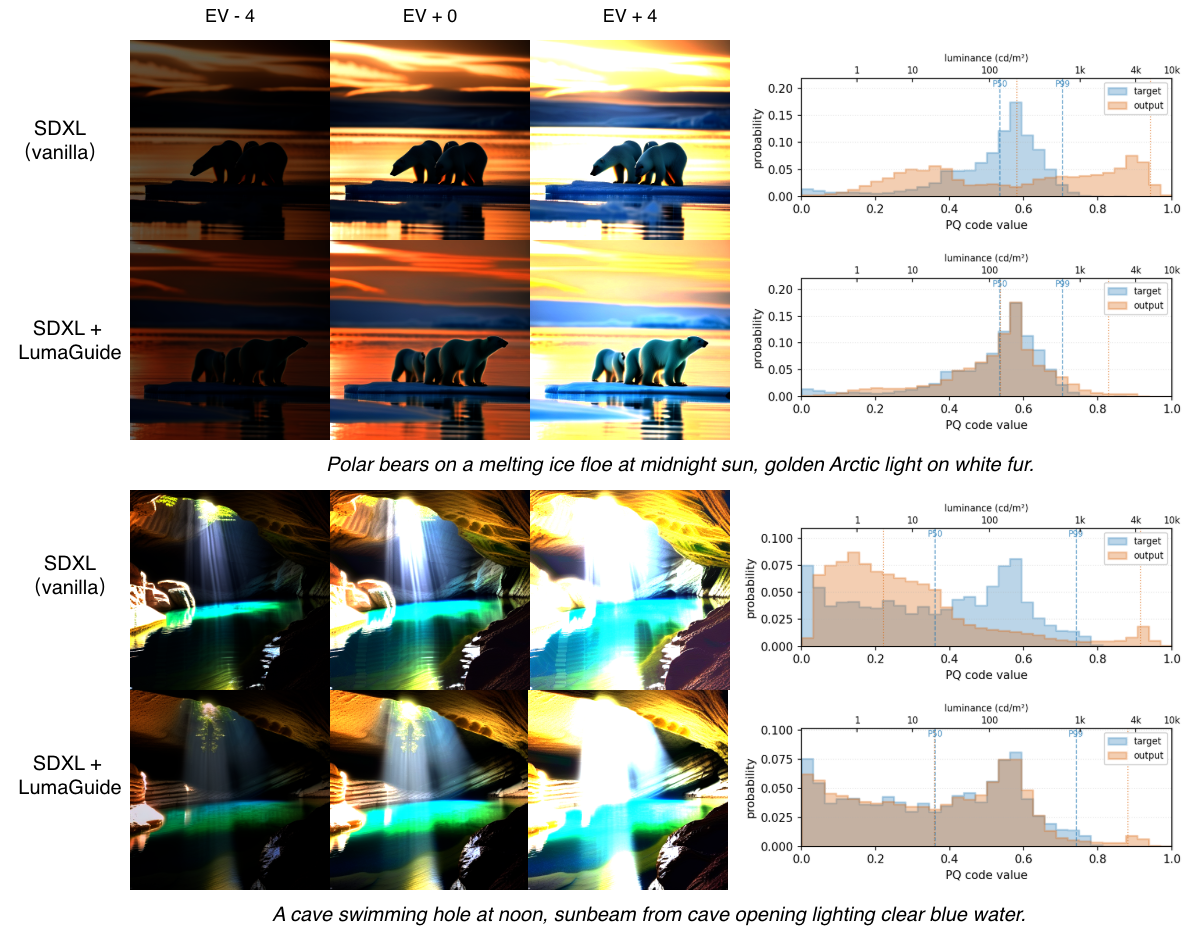}
    \caption{Qualitative results on SDXL and luminance distribution matching at \(s_0=50\).}
    \label{fig:sdxl_qual}
\end{figure}

\begin{table}[t]
\centering
\small
\caption{
Effect of guidance scale $s_0$ on distribution alignment for SDXL.
}
\label{tab:sdxl_s0_sweep}
\begin{tabular}{lccccc}
\toprule
$s_0$ & uW1 $\downarrow$ & p50$_{\text{dist}}$ $\downarrow$ & p99$_{\text{dist}}$ $\downarrow$ & nits$_{50}$ & nits$_{99}$ \\
\midrule
0   & 4.140 & 0.124 & 0.236 & 40.66 & 6349.08 \\
10  & 3.291 & 0.092 & 0.229 & 28.30 & 6267.76 \\
\rowcolor{green!12}
50  & 1.470 & 0.032 & 0.200 & 34.92 & 5817.46 \\
100 & 0.907 & 0.022 & 0.168 & 38.52 & 5098.24 \\
200 & 0.689 & 0.017 & 0.187 & 38.37 & 6141.69 \\
\bottomrule
\end{tabular}
\end{table}

We next evaluate \textbf{LumaGuide} on SDXL, which differs substantially from Flux and SD3 in both architecture and scheduler. Results are reported in Table~\ref{tab:all_baseline_results}, with qualitative examples in Figure~\ref{fig:sdxl_qual}. Despite these differences, the same qualitative behavior is observed: increasing guidance strength improves distribution alignment (Table~\ref{tab:sdxl_s0_sweep}) but may introduce mild artifacts. We select \(s_0=50\) as the operating point for SDXL because larger scales improve global distribution metrics but introduce more visible artifacts and do not fully resolve the highlight-tail gap.

However, two notable differences emerge. First, the selected guidance scale is significantly smaller than in rectified-flow models, reflecting differences in the noise schedule and step size. Second, SDXL exhibits a persistent gap in highlight alignment, with higher $p99_{\text{dist}}$ compared to Flux and SD3. 

\subsection{Details of the Text-to-Histogram Regressor}
\label{app:regressor}

\paragraph{Training data.}
We use a large-scale HDR user-generated content (UGC) dataset, Beyond8bits~\cite{Saini_2026_Beyond8Bits, Saini_2025_ICIP_CHUG, Saini_2026_WACV_BrightRate}. From each video, we uniformly sample four frames, resulting in about 24,000 unique HDR images. All frames are stored in BT.2020 color space with ST.2084 (PQ) encoding, represented as RGB values in $[0,1]$. 

For each frame, we generate a free-form natural language caption using Qwen2.5-VL-7B-Instruct~\cite{qwen25}, prompting it to describe scene content, lighting conditions, and dominant materials. The resulting (caption, frame) pairs provide supervision for learning a mapping from text to luminance distributions.

\paragraph{Target representation.}
For each frame, we compute a 32-bin soft histogram of PQ luminance over the $[0,1]$ range. Soft binning is implemented using a Gaussian kernel with standard deviation $\sigma = 0.5 / K$ (half a bin width, where $K=32$). This yields a normalized probability vector representing the luminance distribution. 

\paragraph{Model architecture.}
The regressor maps a text caption to a luminance histogram. Captions are first encoded using a frozen CLIP ViT-L/14 text encoder, producing a 768-dimensional embedding. This embedding is then passed through a three-layer MLP with dimensions $768 \rightarrow 256 \rightarrow 256 \rightarrow 32$, using GELU activations. A softmax layer is applied at the output to produce a normalized distribution. The model contains approximately 270K trainable parameters, with the text encoder kept frozen.

\paragraph{Training protocol.}
We train the regressor by minimizing the KL divergence between the target and predicted histograms, $\mathrm{KL}(\text{target} \,\|\, \text{prediction})$, averaged over the batch. Optimization is performed using AdamW with learning rate $10^{-3}$ and weight decay $10^{-4}$, for 50 epochs with batch size 256.

To ensure proper generalization, we split the dataset at the video level (i.e., frames from the same video are assigned to the same partition), using a 90\% / 5\% / 5\% train/validation/test split. The final model is selected based on validation KL divergence.

\paragraph{Results.}
On the held-out test set (video-disjoint), the regressor achieves a KL divergence of 0.346, close to the best validation performance (0.338), indicating good generalization to unseen content. These results suggest that text descriptions provide sufficient signal to predict coarse luminance distributions, enabling flexible text-driven specification of target histograms.

\subsection{Subjective Study}
\label{app:study}
\paragraph{Setup}

We conducted a four-way ranking subjective study comparing our method against three state-of-the-art HDR generation baselines: BracketDiffusion~\cite{bemana2025bracket}, LEDiff~\cite{wang2025lediff}, and X2HDR~\cite{wu2026x2hdr}. The study was conducted on consumer-level HDR displays. Stimuli were presented through a Chrome web browser in full-screen mode.

\paragraph{Stimuli}

We used the generated output images spanning diverse HDR-relevant lighting scenarios, including specular highlights, low-light scenes with point sources, high-contrast outdoor environments, dusk/twilight conditions, and indoor mixed lighting. For each prompt, we generated outputs from all four methods using a fixed random seed.

To enable fair perceptual comparison while controlling for absolute luminance differences across methods, all outputs were normalized following the convention of X2HDR: each image was scaled such that its 99.5th-percentile luminance mapped to \(4000~\mathrm{cd/m^2}\). Images were encoded as 10-bit AVIF and verified for correct HDR metadata before the study.

\paragraph{Protocol}

We adopted a four-alternative full-ranking protocol. On each trial, four images (one per method) were presented simultaneously in a randomized $2 \times 2$ grid, and observers ranked them from best (1) to worst (4) based on overall HDR quality, considering realism of bright highlights, detail preservation in shadows, naturalness of color and contrast, and absence of artifacts (e.g., banding, blown highlights, color shifts). Observers were explicitly instructed to disregard text-prompt alignment and focus solely on HDR image quality.

Each observer completed 3 practice trials (data discarded) followed by 40 main trials, with image positions randomized per trial to control for position bias. A 600~ms mid-gray inter-trial blank was inserted between trials to reset visual adaptation, and a 30-second mid-session rest was offered after the 20th trial. Sessions lasted approximately 15 minutes depending on observer deliberation time.

\paragraph{Observers}

We recruited 15 observers with normal or corrected-to-normal vision and normal color vision. Expertise levels ranged from naive viewers to HDR/imaging experts.

\paragraph{Analysis}

For each trial, the 4-way ranking yields $\binom{4}{2} = 6$ pairwise outcomes. We converted the 600 rankings into \(600\times6=3600\) pairwise outcomes and aggregated them into a \(4\times4\) pairwise preference matrix. Method scores were estimated via maximum-likelihood Bradley--Terry and converted to Just Objectionable Difference (JOD) units, where 1 JOD corresponds to a 75\% pairwise preference threshold. 95\% confidence intervals were computed via 2000-iteration bootstrap resampling over trials.

\paragraph{Results.}
Table~\ref{tab:subjective_overall} summarizes the overall ranking results. 
\textbf{LumaGuide} achieves the best subjective performance, with the lowest mean rank, highest win rate, highest Top-2 rate, lowest worst-rank rate, and highest JOD score among all methods. 
Compared with X2HDR, the strongest baseline, \textbf{LumaGuide} obtains a higher win rate and a lower worst-rank rate, indicating fewer severe failure cases and more consistently preferred HDR outputs across diverse content.

\begin{table}[t]
\centering
\small
\caption{Overall subjective ranking results. Lower mean rank and worst-rank rate are better; higher win rate, Top-2 rate, and JOD are better.}
\label{tab:subjective_overall}
\begin{tabular}{lccccc}
\toprule
Method & Mean rank \(\downarrow\) & Win rate \(\uparrow\) & Top-2 rate \(\uparrow\) & Worst rate \(\downarrow\) & JOD [95\% CI] \\
\midrule
\textbf{LumaGuide} & \textbf{1.80} & \textbf{50.2\%} & \textbf{77.7\%} & \textbf{7.7\%} & \textbf{+0.75 [+0.65, +0.85]} \\
X2HDR & 2.09 & 32.0\% & 71.7\% & 12.8\% & +0.43 [+0.34, +0.52] \\
BracketDiffusion & 2.79 & 12.5\% & 32.5\% & 24.5\% & \(-0.30\) [\(-0.39\), \(-0.21\)] \\
LEDiff & 3.32 & 5.3\% & 17.7\% & 55.2\% & \(-0.88\) [\(-0.99\), \(-0.79\)] \\
\bottomrule
\end{tabular}
\end{table}



\subsection{Additional Video Results}
\label{app:video}
\begin{figure}
    \centering
    \includegraphics[width=\linewidth]{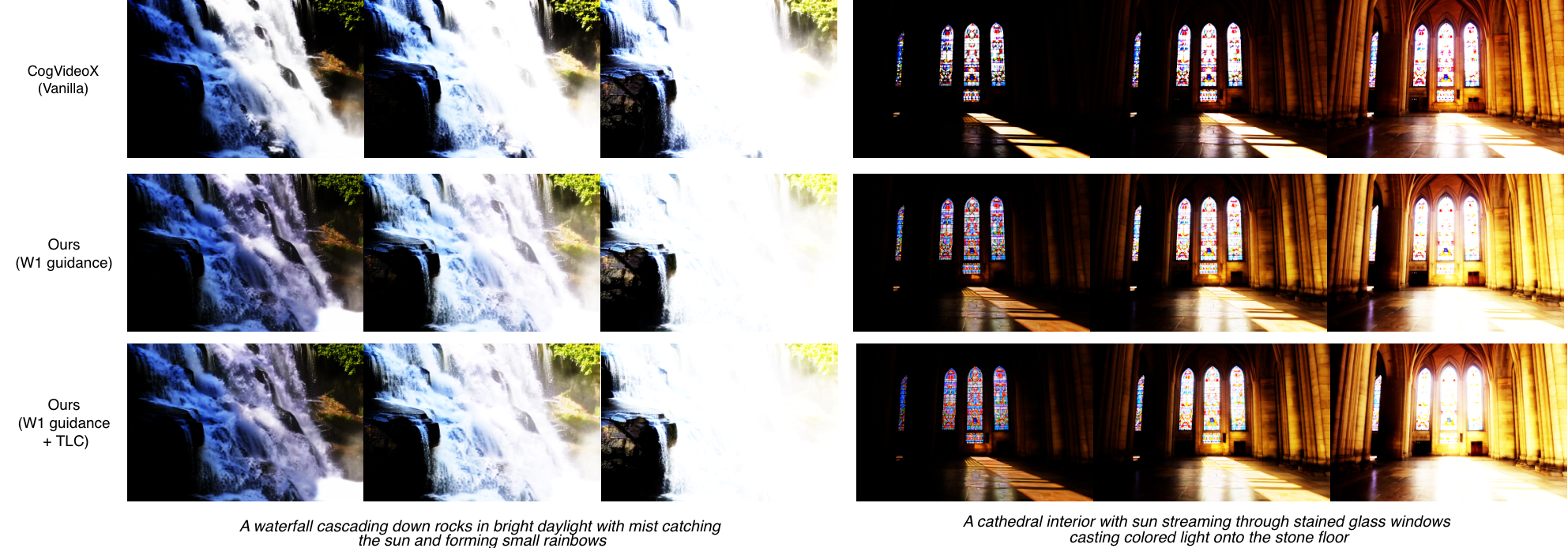}
    \caption{Qualitative video examples comparing vanilla CogVideoX, \(W_1\)-guided LumaGuide, and LumaGuide with TLC. LumaGuide improves highlight control and luminance consistency while preserving scene content.}
    \label{fig:video_qual}
\end{figure}

\begin{figure}
    \centering
    \includegraphics[width=\linewidth]{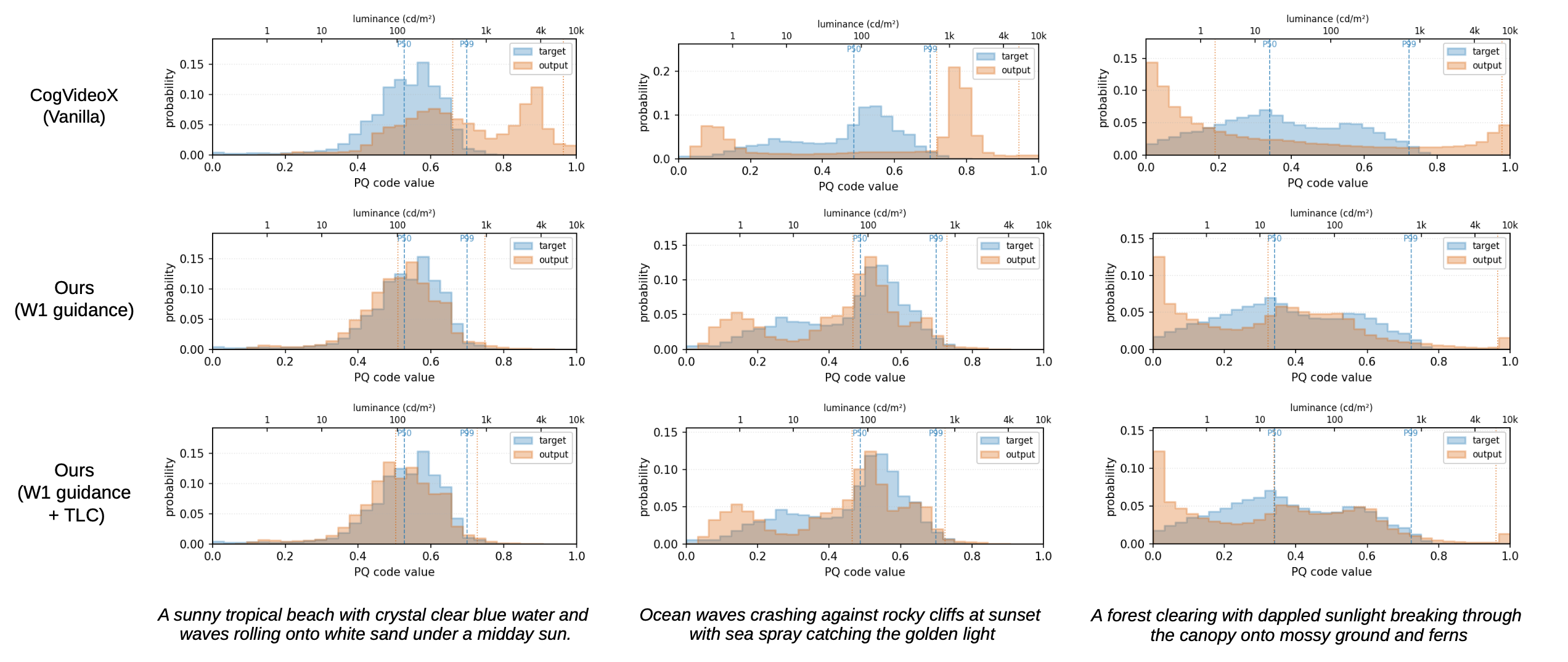}
    \caption{
    Distribution alignment of generated video under three representative prompts. 
    Top row: Vanilla CogVideoX~\cite{cogvideox} produces over-saturated PQ histograms that are poorly aligned with the target distribution (blue), with excessive mass concentrated in high-luminance regions. 
    Middle row: Our W$_1$ guidance shifts the output distributions (orange) toward the target, improving alignment across both mid-tones and high-intensity regions while preserving semantic content. 
    Bottom row: Adding temporal luminance coherence (TLC) maintains distribution alignment. 
    }
    \label{fig:video_hist}
\end{figure}

\begin{figure}
    \centering
    \includegraphics[width=\linewidth]{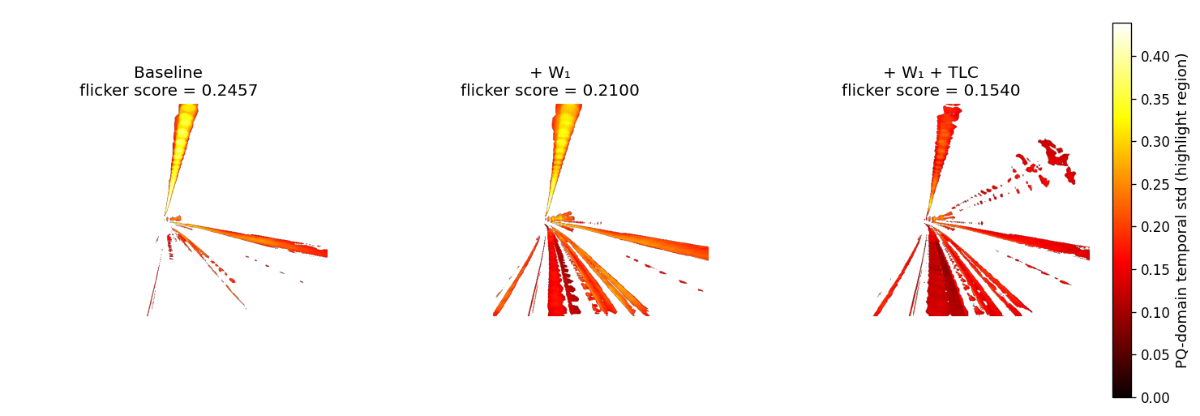}
    \caption{
    Per-pixel temporal standard deviation in highlight regions (PQ $> 0.6$) for the prompt \textit{A subway tunnel with a train approaching and bright headlights brightening the dark concrete walls}. 
    The baseline exhibits substantial flicker (yellow regions, score 0.246). 
    $W_1$ guidance alone marginally reduces flicker (0.210). 
    Adding TLC further suppresses flicker, as the temporal-variance term explicitly penalizes per-pixel variation in high-luminance regions.
    }
    \label{fig:video-flicker}
\end{figure}

We evaluate the generality of LumaGuide on video generation using CogVideoX-5B, a 5B-parameter open-source text-to-video latent diffusion model. The evaluation uses 30 prompts covering different types of scenes.

For each prompt, we generate a 33-frame video at $480 \times 720$ resolution (8 FPS, $\sim$4 seconds) with 50 diffusion steps. We compare three settings: (a) the vanilla CogVideoX baseline, (b) adding PQ-domain $W_1$ guidance ($s_0 = 3000$, constant schedule, perceptual-log formulation), and (c) adding temporal luminance coherence (TLC) on top of $W_1$ ($\lambda_f = \lambda_c = 10$). Target luminance histograms are predicted using our text-to-histogram regressor. All experiments are conducted on a single NVIDIA GH200 GPU.

Since there is no widely adopted reference-free metric for HDR video generation, we rely on two simple, ground-truth-free measures. First, uW$_1$ measures the Wasserstein-1 distance between the generated PQ luminance histogram (aggregated across frames) and the predicted target distribution. Second, we compute flicker variance as the mean temporal variance of per-pixel luminance in highlight regions (PQ $> 0.6$), which captures instability in bright areas.

Figure~\ref{fig:video_hist} shows that the baseline produces poorly aligned luminance distributions, with a large portion of mass pushed toward high PQ values (near 1.0), which corresponds to over-saturated highlights. Adding $W_1$ guidance consistently shifts the distributions toward the target across all prompts, improving both median (P50) and high-percentile (P99) alignment, in line with the drop in uW$_1$ in Table~\ref{tab:video_main}.

This effect is also visible in Figure~\ref{fig:video_qual}. The baseline tends to over-expose bright regions and produces noticeable frame-to-frame fluctuations. With $W_1$ guidance, highlight structure becomes more controlled and luminance is redistributed more evenly, while the overall scene content is preserved.

Figure~\ref{fig:video-flicker} highlights the temporal behavior. The baseline shows large high-variance regions in bright areas, indicating strong flicker. $W_1$ alone reduces this only slightly. When TLC is added, these high-variance regions are largely suppressed, leading to visibly more stable highlights across frames. This is consistent with the reduction in flicker variance from 0.025 to 0.016.

Importantly, adding TLC does not harm distribution alignment, as uW$_1$ remains essentially unchanged. Overall, $W_1$ mainly improves the spatial luminance distribution, while TLC addresses temporal instability, and the two components work well together.

\begin{table}[t]
\centering
\small
\caption{Quantitative results on CogVideoX-5B.}
\label{tab:video_main}
\begin{tabular}{lcccc}
\toprule
Method & uW$_1$ $\downarrow$ & P50 gap $\downarrow$ & P99 gap $\downarrow$ & Flicker var $\downarrow$ \\
\midrule
Vanilla CogVideoX & 6.85 $\pm$ 2.83 & 0.250 & 0.245 & 0.033 \\
+ $W_1$ guidance  & 4.03 $\pm$ 3.08 & 0.159 & 0.155 & 0.025 \\
+ $W_1$ + TLC     & \textbf{3.98} $\pm$ 3.12 & 0.160 & \textbf{0.154} & \textbf{0.016} \\
\bottomrule
\end{tabular}
\end{table}

\begin{table}[t]
\centering
\small
\caption{Pairwise improvements across prompts.}
\label{tab:video_pairwise}
\begin{tabular}{lcc}
\toprule
Comparison & Mean $\Delta$ flicker var & Median $\Delta$ \\
\midrule
baseline $\rightarrow$ + $W_1$         & 0.008 & 0.007 \\
+ $W_1$ $\rightarrow$ + $W_1$ + TLC    & 0.009 & 0.006 \\
baseline $\rightarrow$ + $W_1$ + TLC   & 0.018 & 0.015 \\
\bottomrule
\end{tabular}
\end{table}

\subsection{Analysis on Evaluation Metrics}
\label{app:metric}

\begin{figure}[t]
\centering
\includegraphics[width=\linewidth]{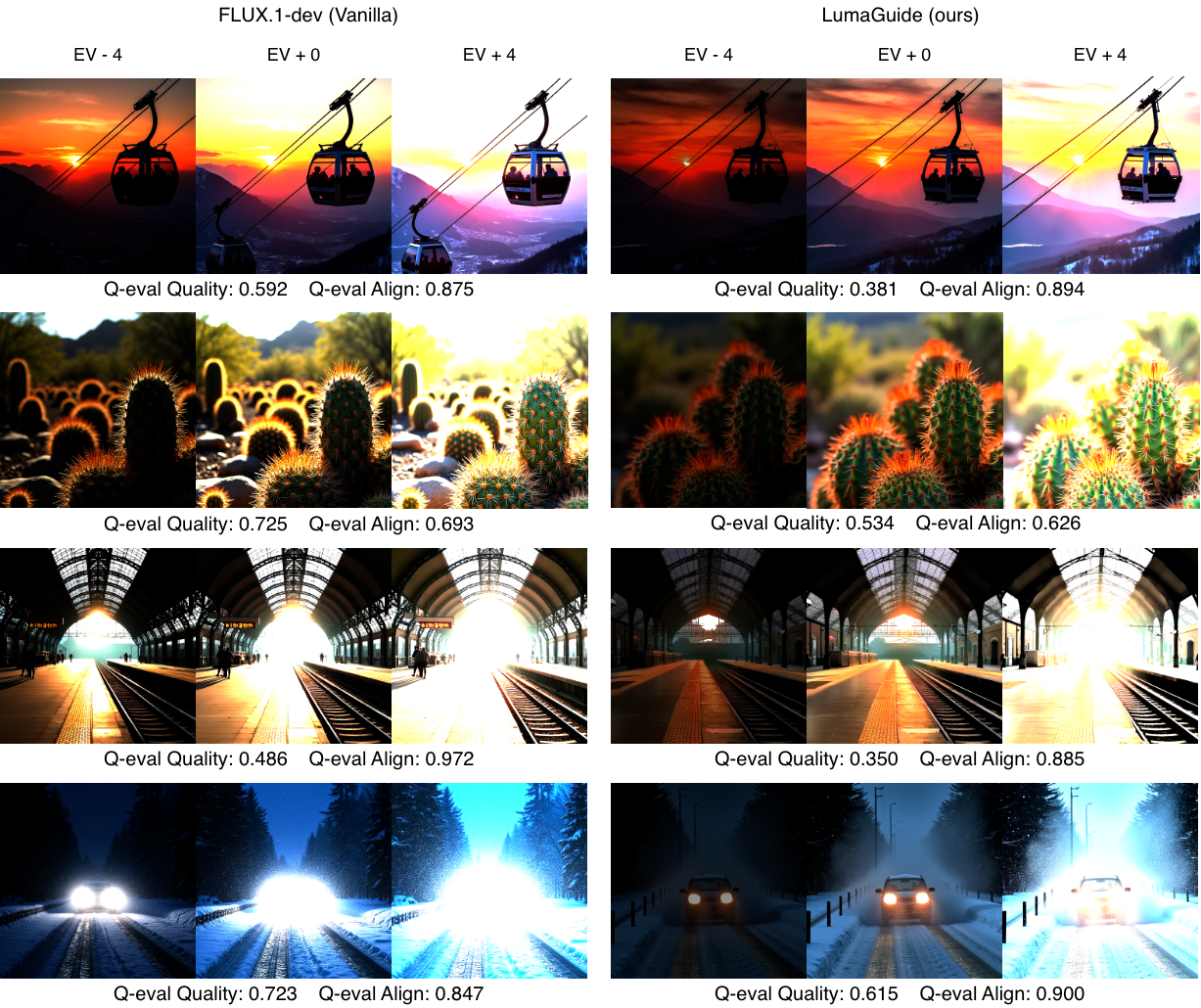}
\caption{
Illustration of the mismatch between Q-Eval quality scores and perceptual HDR quality. Images with stronger HDR luminance structure may receive lower Q-Eval quality scores despite improved distributional alignment.
}
\label{fig:metric_mismatch}
\end{figure}

We measure dynamic range following~\cite{wu2026x2hdr}. For each generated image, we first compute BT.2020 weighted luminance in linear HDR space:
\begin{equation}
Y(x, y) = 0.2627R + 0.6780G + 0.0593B,
\end{equation}
where $Y$ is measured in \(\mathrm{cd/m^2}\). To suppress isolated decoder noise that may artificially inflate luminance extremes, we apply Gaussian smoothing with $\sigma = 3$ pixels before computing statistics.

We then estimate the effective dynamic range using the 0.5th and 99.5th luminance percentiles, which robustly capture the visible luminance span while excluding numerical outliers. Following common HDR evaluation practice, the lower percentile is clamped to \(0.05~\mathrm{cd/m^2}\) to avoid unstable ratios near black.The final metric is defined as:
\begin{equation}
\mathrm{DR}_{\mathrm{stops}}
=
\log_2
\left(
\frac{
\tilde{Y}_{99.5}
}{
\max(\tilde{Y}_{0.5}, 0.05)
}
\right),
\end{equation}
where $\tilde{Y}$ denotes the Gaussian-smoothed luminance.

We also adopt the Q-Eval quality and alignment scores following prior HDR generation work. However, we observe a systematic discrepancy between Q-Eval quality scores (on HDR) and perceptual assessment in our setting.

In our framework, generated images are represented in PQ space, and evaluation is performed directly on PQ-encoded signals following prior HDR generation work. As a result, the input distribution to Q-Eval differs from its expected input regime, which is closer to SDR-like image statistics. This introduces a mismatch between the evaluation metric and the target signal domain. We observe that increasing the guidance scale generally decreases Q-Eval quality scores. At large guidance values, this behavior is expected, as overly strong guidance introduces visible artifacts such as banding and unnatural textures. However, even within the moderate guidance regime, where visual inspection indicates improved HDR characteristics and better alignment with the target luminance distribution, Q-Eval quality scores still decrease.

This suggests that Q-Eval quality is sensitive to deviations from its learned SDR prior, rather than to the correctness of HDR luminance distributions. In particular, the redistribution of luminance toward HDR distribution, which is essential for HDR generation, may be penalized as a distribution shift from SDR statistics.

In contrast, the Q-Eval alignment score remains relatively stable across guidance scales. This is consistent with its design, as it primarily measures semantic consistency between the generated image and the text prompt, and is less sensitive to luminance distribution.

These observations highlight a limitation of current evaluation metrics when applied to HDR generation. While Q-Eval provides a useful proxy for perceptual quality, it does not fully capture the correctness of HDR luminance structure. Therefore, we complement it with distribution-based metrics and qualitative analysis in our evaluation. As shown in Figure~\ref{fig:metric_mismatch}, this discrepancy is clearly visible: images with more realistic HDR behavior receive lower Q-Eval quality scores.

\subsection{Failure Cases and Limitations}

We analyze two representative failure modes of \textbf{LumaGuide} that arise from the interaction between distribution-level guidance and the underlying generative model.

\begin{figure}
    \centering
    \includegraphics[width=\linewidth]{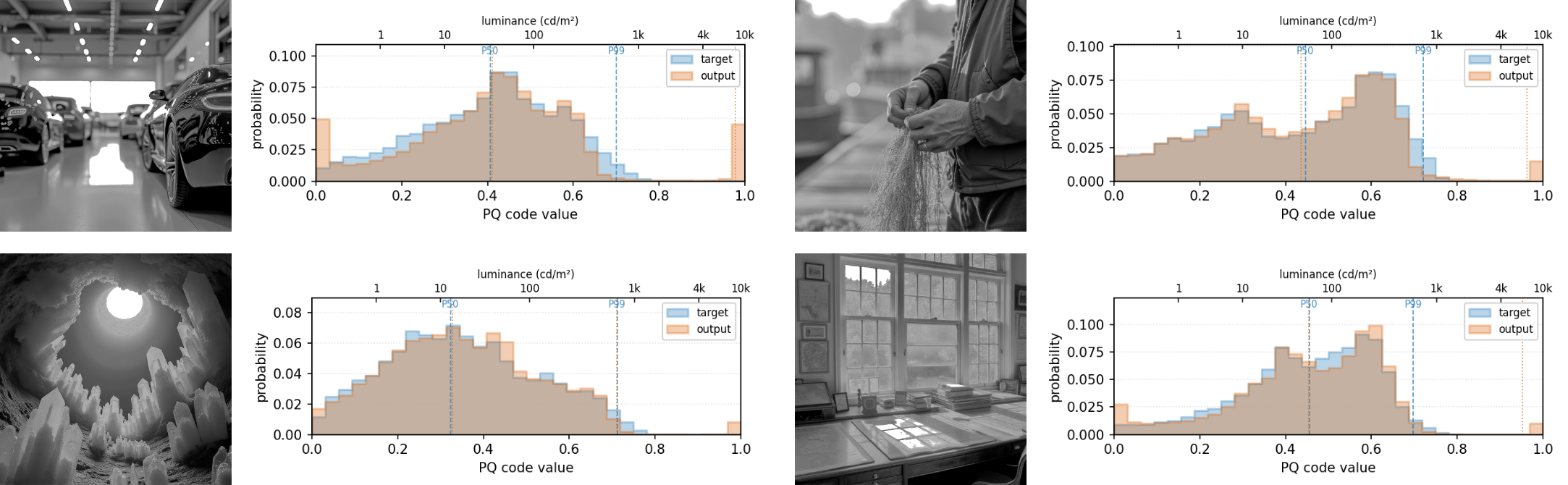}
    \caption{Failure case involving saturation-induced unguidable regions.}
    \label{fig:failure1}
\end{figure}

\begin{figure}
    \centering
    \includegraphics[width=\linewidth]{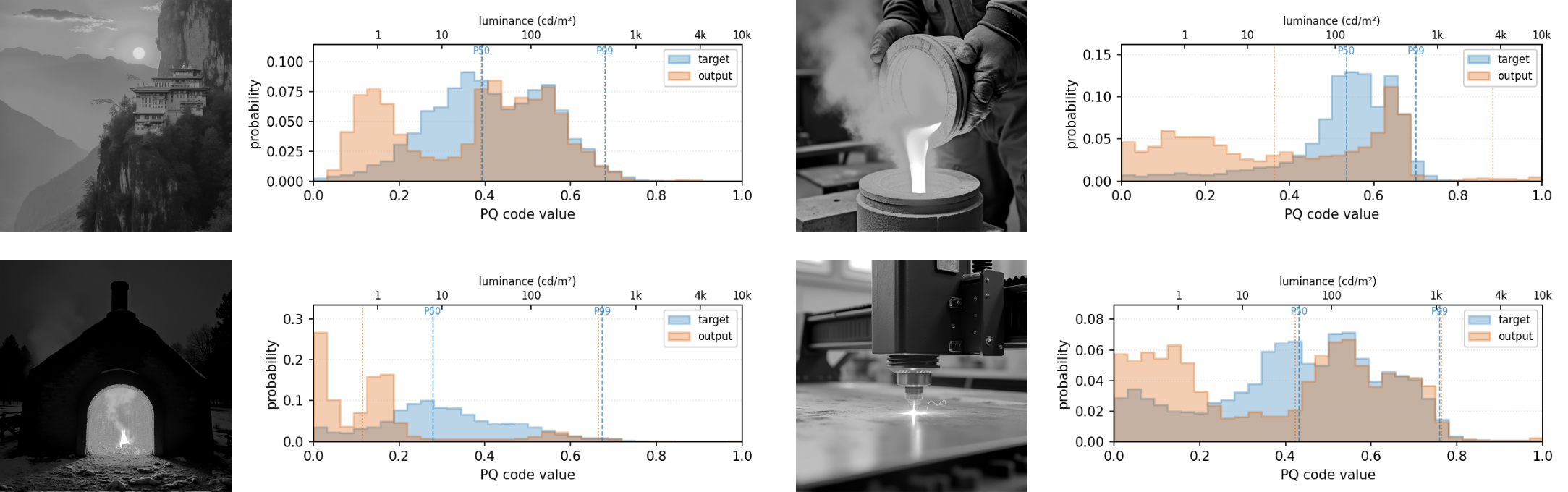}
    \caption{Failure case involving target distribution infeasibility.}
    \label{fig:failure2}
\end{figure}

\paragraph{Saturation-Induced Unguidable Regions.}
In scenes dominated by strong light sources (e.g., sky, sun, specular highlights), we observe that some pixels remain saturated at the maximum PQ value and cannot be reduced by guidance shown in Figure~\ref{fig:failure1}. Increasing the guidance strength $s_0$ often darkens surrounding regions while leaving saturated pixels unchanged, resulting in contrast exaggeration and halo artifacts. This behavior is caused by gradient vanishing at the decoder output: once pixels are clamped to the upper bound, their gradients become zero and no longer propagate back to the latent variables. As a result, the Wasserstein objective only affects neighboring unsaturated pixels, creating a luminance imbalance. This limitation prevents full alignment with the target distribution and leads to residual errors in high-percentile statistics (e.g., $p99$). Addressing this issue may require smoother decoding functions or HDR-aware architectures that preserve gradients near saturation.

\paragraph{Target Distribution Infeasibility.}
We also observe cases, shown in Figure~\ref{fig:failure2},  where the generated images cannot fully match the target histogram, particularly for prompts corresponding to low-dynamic-range scenes (e.g., indoor or low-contrast environments). In such cases, the model converges to a compromise solution that only partially aligns with the target distribution. This reflects a feasibility limitation: the pretrained diffusion model defines a manifold of plausible images, and not all luminance distributions are achievable within this space. When the target distribution is incompatible with the scene semantics, the guidance objective becomes partially infeasible, resulting in persistent distribution mismatch. This suggests that effective distribution shaping requires compatibility between the target statistics and the underlying generative prior, and motivates adaptive or content-aware target specification.

\subsection{Future Work}

Our analysis highlights several directions for improving distribution-level guidance.

First, saturation-induced failures suggest the need for HDR-aware decoding mechanisms that preserve gradients near extreme luminance values. Addressing this limitation is critical for accurately controlling high-intensity regions.

Second, the effectiveness of distribution shaping depends on the compatibility between the target distribution and the generative prior. Future work may explore adaptive or content-aware target specification to ensure feasibility during sampling.

Finally, we observe that existing evaluation metrics for HDR generation remain limited. Most perceptual metrics are designed for SDR content or are insensitive to luminance distribution, making them poorly aligned with HDR characteristics. Developing metrics that capture both perceptual quality and distributional fidelity in HDR space is an important direction for future research.

\subsection{Broader Impacts and Safeguards}

\paragraph{Broader Impacts.}
This work improves controllable HDR image and video generation through training-free luminance distribution shaping. Potential positive impacts include lowering the computational cost of HDR content creation and enabling more accessible HDR workflows for creative applications such as digital media, film, and immersive visualization.

At the same time, improved HDR realism may increase risks already associated with generative media, including deceptive or misleading synthetic content. Since HDR imagery can produce more realistic lighting and highlight behavior, generated outputs may appear more visually convincing. However, our method does not introduce a new generative backbone or new semantic generation capabilities; it operates only as a test-time guidance mechanism on existing diffusion models.

\paragraph{Safeguards.}
Our work does not release a new large-scale generative model. The proposed method is an inference-time guidance technique applied to publicly available pretrained backbones. The lightweight text-to-histogram regressor predicts only coarse luminance distributions and cannot independently generate images.

We do not release any private user metadata or identity-related supervision. The proposed method focuses solely on luminance distribution control and does not introduce mechanisms for identity imitation, biometric analysis, or targeted manipulation.

\end{document}